\documentclass{article}

\usepackage{microtype}
\usepackage{graphicx}
\usepackage{subcaption}
\usepackage{booktabs} % for professional tables

\usepackage{hyperref}

\makeatletter
\usepackage[preprint]{style/icml2026}
\makeatother

\usepackage{amsmath}
\usepackage{amssymb}
\usepackage{mathtools}
\usepackage{amsthm}

\usepackage{bm}
\usepackage{enumitem}
\usepackage[table]{xcolor}
\usepackage{multirow}
\usepackage{multirow}
\usepackage{booktabs,longtable}
\usepackage[most]{tcolorbox}
\tcbuselibrary{skins} % 用于更高级的皮肤

\usepackage[capitalize,noabbrev]{cleveref}

\theoremstyle{plain}
\newtheorem{theorem}{Theorem}[section]

\newtheorem{lemma}[theorem]{Lemma}

\theoremstyle{definition}
\newtheorem{definition}[theorem]{Definition}

\theoremstyle{remark}
\newtheorem{remark}[theorem]{Remark}

\usepackage[textsize=tiny]{todonotes}

\begin{document}

\twocolumn[
  \icmltitle{RACER: Risk-Aware Calibrated Efficient Routing for Large Language Models}



  \icmlsetsymbol{equal}{*}

  % \begin{icmlauthorlist}
  %   \icmlauthor{Firstname1 Lastname1}{equal,yyy}
  %   \icmlauthor{Firstname2 Lastname2}{equal,yyy,comp}
  %   \icmlauthor{Firstname3 Lastname3}{comp}
  %   \icmlauthor{Firstname4 Lastname4}{sch}
  %   \icmlauthor{Firstname5 Lastname5}{yyy}
  %   \icmlauthor{Firstname6 Lastname6}{sch,yyy,comp}
  %   \icmlauthor{Firstname7 Lastname7}{comp}
  %   %\icmlauthor{}{sch}
  %   \icmlauthor{Firstname8 Lastname8}{sch}
  %   \icmlauthor{Firstname8 Lastname8}{yyy,comp}
  %   %\icmlauthor{}{sch}
  %   %\icmlauthor{}{sch}
  % \end{icmlauthorlist}



\begin{icmlauthorlist}
  \icmlauthor{Sai Hao}{sustech}
  \icmlauthor{Hao Zeng}{sustech}
  \icmlauthor{Hongxin Wei}{sustech}
  \icmlauthor{Bingyi Jing}{cuhksz,slai}
\end{icmlauthorlist}

\icmlaffiliation{sustech}{Southern University of Science and Technology, China}
\icmlaffiliation{cuhksz}{The Chinese University of Hong Kong, Shenzhen, China}
\icmlaffiliation{slai}{Shenzhen Loop Area Institute, China}

% co-Corresponding authors (list both)
\icmlcorrespondingauthor{Bingyi Jing}{bingyijing@cuhk.edu.cn}
\icmlcorrespondingauthor{Hongxin Wei}{weihx@sustech.edu.cn}


  \vskip 0.3in
]



\printAffiliationsAndNotice{}  % no special notice (required even if empty)

\begin{abstract}
Efficiently routing queries to the optimal large language model (LLM) is crucial for optimizing the cost-performance trade-off in multi-model systems.  
However, most existing routers rely on single-model selection, making them susceptible to misrouting.
In this work, we formulate LLM routing as the $\alpha$-VOR problem to minimize expected set size while controlling the misrouting risk, 
and propose a novel method -- RACER, extending base routers to output model sets that can be subsequently aggregated for improved output.
In particular, RACER constructs nested model sets via augmented scoring and utilizes finite-sample concentration bounds to calibrate a threshold that allows for both variable set sizes and abstention. 
We theoretically prove that RACER achieves rigorous distribution-free risk control on unseen test data in a post-hoc and model-agnostic manner.
Extensive experiments verify our theoretical guarantees and demonstrate that RACER consistently enhances downstream accuracy across a wide range of benchmarks. 
\end{abstract}

\section{Introduction}

% 多llm系统的介绍和naive的模型选择方法介绍-cascade
Large language models (LLMs) are increasingly deployed not as stand-alone systems, but as components of larger \emph{multi-model systems}, where multiple LLMs with different capabilities and costs coexist. 
Due to variations in data and architectures, these LLMs often exhibit complementary strengths and weaknesses across different domains\cite{chen2024routerdc}. In such settings, a naive strategy is to invoke all candidate models for every query and aggregate their responses via scoring or voting \cite{li2024more,wang2022self,jiang2023llm}. 
While this paradigm can achieve strong performance, its computational cost is often too high in practice. 
Thus, it is essential to determine the invoked LLM(s) for each incoming query, which highlights the cost-performance trade-off in multi-model systems.

% 引入router研究，说明现有router方法效果存在局限，不能有效地挑出最好的llm
Recent research proposed \emph{LLM routers} to predict the most suitable candidate for each query by training a lightweight model, without calling all LLMs~\cite{huang2025routereval,chen2024routerdc,lu2024routing}. 
Yet, on real-world benchmarks, even state-of-the-art routers can misrank candidates and select the wrong LLM, leading to a significant performance drop compared to the ideal selection\cite{huang2025routereval}. 
To mitigate this mismatch, a natural strategy is to expand the selection to a subset of top-ranked candidates.
However, existing subset routing methods often rely on heuristic size controls~\cite{jiang2023llm}, 
which lack coverage guarantees and potentially introduce noisy output from incorrect models that degrades the final decision~\cite{vishwakarma2025prune}.
This limitation raises a pivotal issue: 
\vspace{-1.5em}
\begin{quote}
\textit{How can we constrain the selection set size while guaranteeing that it contains a correct model?}
\end{quote}
\vspace{-0.3em}

In this work, we propose \textit{\textbf{R}isk-\textbf{A}ware \textbf{C}alibrated \textbf{E}fficient \textbf{R}outing} (RACER), a novel post-hoc paradigm with guaranteed risk control. 
We define the routing \textit{risk} as the probability of excluding all optimal LLMs, and formulate LLM routing as the \textit{$\alpha$-Valid Optimal Routing} ($\alpha$-VOR) problem (See Definition~\ref{def:problem}): minimize the expected model set size while bounding the \textit{risk} below a user-specified level $\alpha$.
RACER employs an augmented scoring mechanism to construct a nested family of model sets, and leverages finite-sample concentration bounds to calibrate the conservativeness threshold, thereby routing queries to a model set satisfying the $\alpha$-VOR constraint. 
This paradigm also supports abstention when no candidate LLM is suitable.
With effective aggregation strategies, we harness the different strengths of the selected LLMs to generate an output superior to single-model selection across diverse domains.
Overall, RACER is lightweight, flexible, and model-agnostic, as it enhances arbitrary black-box routers without retraining.

Theoretically, we provide distribution-free guarantees, ensuring that the risk is controlled below $\alpha$ assuming exchangeability (Theorem~\ref{thm:risk_control}). 
We also establish a risk lower bound (Theorem~\ref{thm:lower_bound}), showing that RACER balances safety and efficiency without being overly conservative. 
These results rely on the nestedness of the prediction sets (Lemma~\ref{lem:nested}) and the monotonicity of the risk (Lemma~\ref{lem:loss}).
 
To verify the validity and effectiveness of RACER, 
we conduct extensive studies on four diverse benchmarks (GSM8K \cite{cobbe2021training}, MMLU \cite{hendrycks2021measuring}, CMMLU \cite{li2024cmmlu}, and ARC-Challenge \cite{clark2018think}) using three distinct base routers and two nonconformity scores over a candidate pool of seven LLMs. 
The results show that RACER achieves rigorous risk control, and consistently improves downstream accuracy with aggregation strategies over single-model selection. 
Specifically, compared with the base routers, our method achieves up to a $4.0\%$ absolute accuracy improvement on individual benchmarks and an average improvement of $3.6\%$ across all tasks. 
Moreover, RACER surpasses the single best-performing candidate LLM by $5.0\%$ on average across all tasks.
Furthermore, extended experiments demonstrate that compared to full-model aggregation, RACER achieves higher accuracy while reducing model calls by up to $58.6\%$.
Our code is available at \url{https://anonymous.4open.science/r/RACER-27A9/}.

% 0121修改
Our contributions are summarized as follows:
\begin{itemize}[itemsep=3pt]
    % \item We formally define the \textbf{$\alpha$-Valid Optimal Routing} ($\alpha$-VOR) problem.
    % We depart from traditional top-1 routing by allowing the router to output a \emph{set} of candidate models, enabling explicit control over the risk of excluding optimal models. 
    % The resulting $\alpha$-VOR formulation captures a principled trade-off between \textit{Risk} and \textit{Size}.
    \item We formulate LLM routing as the \textbf{$\alpha$-VOR problem}, establishing a principled framework to optimize the \textit{cost-performance trade-off} by minimizing the expected model set size while controlling the misrouting risk.
    % \item We propose \textbf{RACER}, a model-agnostic post-hoc paradigm to solve the $\alpha$-VOR problem. RACER enhances arbitrary scoring-based routers by establishing rigorous theoretical guarantees on risk control, all without modifying the base router's architecture or requiring retraining. Furthermore, we empirically demonstrate its robustness by verifying that RACER effectively controls the desired risk level across multiple datasets using various nonconformity scores.
    \item We propose \textbf{RACER}, a novel post-hoc paradigm that transforms single-model selection into calibrated set prediction. It is compatible with any base router, supports empty sets, and consistently improves downstream accuracy through aggregation.
    % \item We demonstrate consistent downstream performance gains across diverse benchmarks.
    % Extensive experiments confirm that RACER universally enhances the downstream accuracy of every corresponding base router by aggregating the outputs of the calibrated model sets.
    \item We establish rigorous \textbf{distribution-free theoretical guarantees}: we prove that RACER controls the misrouting risk on unseen queries at the user-specified level $\alpha$, and provide a matching risk lower bound, showing that the method achieves statistical efficiency without being overly conservative.
\end{itemize}

\section{Background}
\label{sec:background}

In this section, we introduce the multi-model routing problem, discuss the risks associated with standard routers, and formally define the \textit{$\alpha$-Valid Optimal Routing} problem.

\subsection{Preliminaries}
% \paragraph{Notation.}
Let $\mathcal{X}$ be the space of input queries and $\mathcal{Y}$ be the answer space. Consider a pool of $K$ candidate LLMs, denoted by $\mathcal{M} = \{M_1, \dots, M_K\}$. 
We use lowercase $\bm{x}$ for a fixed query and uppercase $\bm{X}$ for a random query.
For a given input query $\bm{x} \in \mathcal{X}$, each LLM $m \in \mathcal{M}$ generates a response $y_m = m(\bm{x}) \in \mathcal{Y}$. We assume the existence of a ground-truth evaluation oracle that determines the correctness of each response. Let $G(\bm{x}) \subseteq \mathcal{M}$ denote the set of \textit{ground-truth LLMs} that generate a correct response for $\bm{x}$. Note that $G(\bm{x})$ may be empty if no candidate LLM in the pool can answer the query correctly.

\paragraph{Multi-model routing.}
A standard router typically relies on a scoring function $f: \mathcal{X} \times \mathcal{M} \to \mathbb{R}$, where $f(\bm{x}, m)$ represents the predicted performance of LLM $m$ for query $\bm{x}$. 
Existing routing methods typically select the single model with the highest score:
\begin{equation*}
    \hat{m}(\bm{x}) = \arg\max_{m \in \mathcal{M}} f(\bm{x}, m).
\end{equation*}
While efficient, the mismatch between predicted rankings and ground truth makes single-model selection susceptible to \textit{misrouting} (i.e., $\hat{m}(\bm{x}) \notin G(\bm{x})$), as even state-of-the-art routers can select a sub-optimal LLM \cite{huang2025routereval}. 
Therefore, expanding the selection to a candidate subset is a natural strategy. This approach improves the likelihood of a correct selection and enables the application of aggregation techniques to yield superior performance~\cite{wang2023selfconsistency,taubenfeld2025confidence}.
However, existing methods often rely on heuristic size controls\cite{jiang2023llm}, lack statistical coverage guarantees and often introduce noise from incorrect models.
To address this challenge, we formalize the LLM routing problem to achieve rigorous risk control while minimizing the expected model set size.

% To address this instability, a natural approach is to prune incorrect models. This is analogous to option pruning in multiple-choice tasks~\cite{vishwakarma2025prune}, where eliminating options narrows the candidate set and improves the likelihood of a correct selection.
% % Furthermore, applying aggregation techniques to the retrieved subset can yield performance superior to random selection~\cite{wang2023selfconsistency,taubenfeld2025confidence}.

\subsection{Problem formulation}
% To address this challenge, we shift the objective from selecting a single best LLM to constructing a router-based function $C: \mathcal{X} \to 2^{\mathcal{M}}$, which maps each query $\bm{x}$ to a model set $C(\bm{x}) \subseteq \mathcal{M}$.
Given a router, we first construct a function $C: \mathcal{X} \to 2^{\mathcal{M}}$, which maps each query $\bm{x}$ to a model set based on router scores.
To ensure the selected subset forms a valid basis for downstream aggregation, we characterize the performance of $C(\cdot)$ along two dimensions: \emph{risk} and \emph{size}.
Specifically, we define the misrouting \textit{loss} as the event where the model set fails to cover any ground-truth LLM:
\begin{equation*}
    \ell(C(\bm{x}), G(\bm{x})) := \mathbf{1}(C(\bm{x}) \cap G(\bm{x}) = \emptyset).
\end{equation*}
The misrouting \emph{risk} is then defined as $R(C) := \mathbb{E}[ \ell(C(\bm{X}), G(\bm{X}))]$, and the expected set \emph{size} is measured by $\mathbb{E}[|C(\bm{X})|]$, the average inference cost (i.e., the number of models invoked per query).

We formulate our goal of minimizing \textit{size} while controlling \textit{risk} as the \textit{$\alpha$-Valid Optimal Routing} ($\alpha$-VOR) problem:
\begin{definition}[$\alpha$-VOR]
	\label{def:problem}
	Given a risk level $\alpha \in (0, 1)$, the goal of $\alpha$-VOR is to find an optimal function $C^*$ that minimizes the expected size of prediction model sets while satisfying the validity constraint:
	\begin{equation} \label{eq:problem}
		C^* = \arg\min_{C} \mathbb{E}[|C(\bm{X})|], \quad \text{subject to} \quad R(C) \le \alpha.
	\end{equation}
\end{definition}
% This formulation provides a principled framework for optimizing the cost-performance trade-off: instead of a single or fixed top-$k$ selection, $C(\cdot)$ adaptively outputs a model set, ensuring the risk is controlled at level $\alpha$.
By strictly bounding the probability of missing the ground truth, $C(\cdot)$ ensures that downstream aggregation mechanisms operate on a set that reliably includes correct answers, thereby translating risk control into superior performance.

In practice, as the data distribution is unknown, solving the optimization in Eq.~\eqref{eq:problem} exactly is infeasible. This requires a data-driven approximation using a finite calibration dataset. 
% We address this challenge in the next section.

\begin{figure*}
    \centering
    \includegraphics[width=0.88\linewidth]{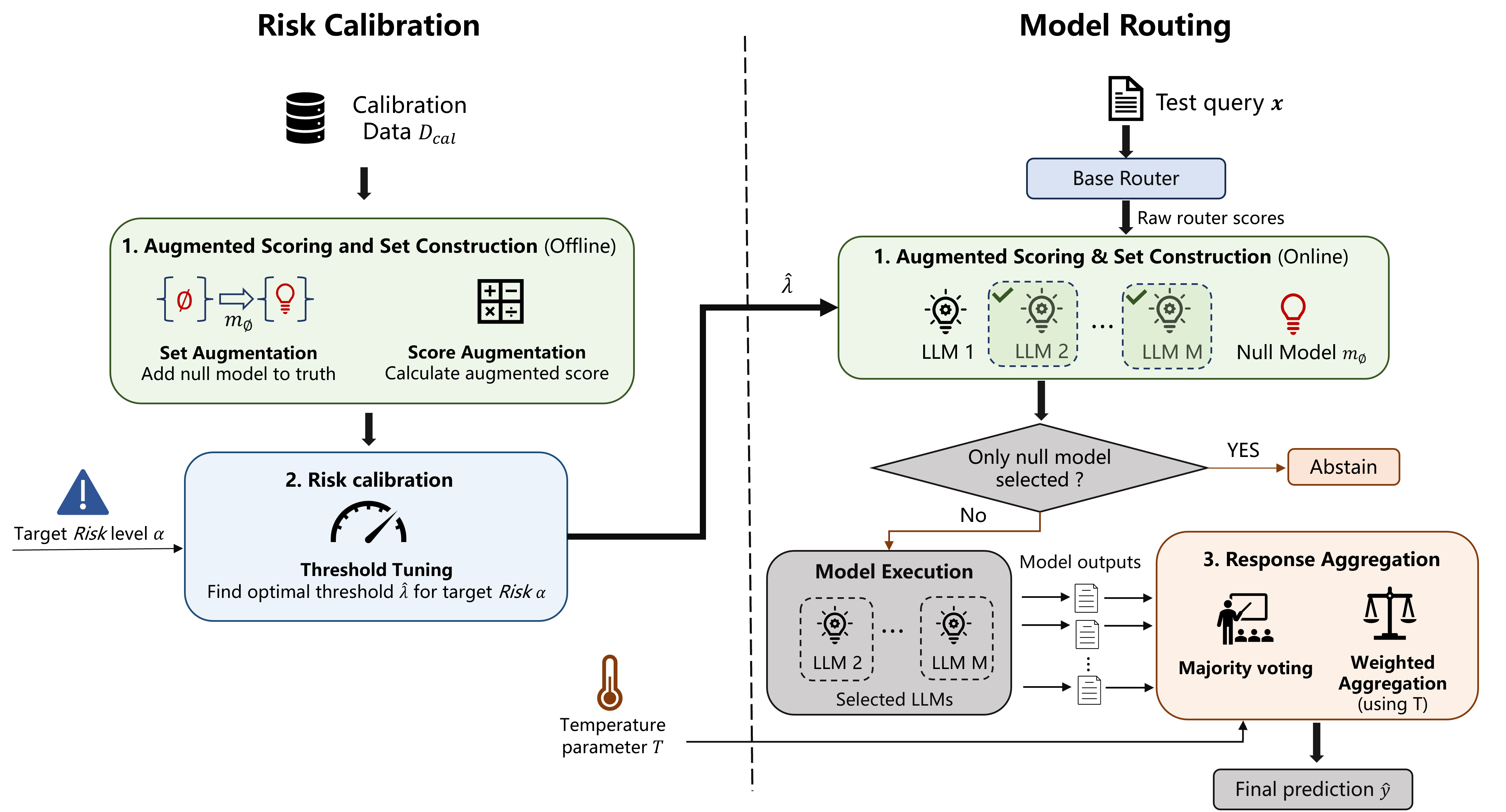}
    \caption{Overview of the RACER paradigm. RACER operates in two phases. \textbf{Risk Calibration (Left):} The calibration module uses a labeled dataset $\mathcal{D}_{\mathrm{cal}}$ and a user-specified risk level $\alpha$. It augments the standard model space $\mathcal{M}$ with a null model $m_\emptyset$ to construct augmented ground truth set $G'$. The threshold $\hat{\lambda}$ is then computed to guarantee risk control. \textbf{Model Routing (Right):} During inference, the paradigm applies the calibrated $\hat{\lambda}$ to the augmented scores of a test query $\bm{x}$. This generates a prediction set $C_{\hat{\lambda}}(\bm{x})$. If the set contains only the null model, the system triggers abstention; otherwise, it proceeds to \textbf{Response Aggregation}, where the outputs of the selected standard LLMs are combined via majority voting or weighted aggregation to produce the final prediction $\hat{y}$.}
    \label{fig:paradigm}
\end{figure*}

\begin{remark}[Interpretation of validity]
    % A distinctive feature of this setting is the potential for queries where no candidate model is correct, i.e., $G(\bm{x})=\emptyset$.
    We consider a routing decision to be valid if it covers a ground truth LLM when one exists, or correctly abstains (i.e., returns the empty set) when no suitable LLM is available. To align this with our loss formulation, we will implicitly assume the existence of a ``null'' model in our proposed paradigm.
\end{remark}

\section{RACER}
\label{sec:method}
In this section, we present \textbf{Risk-Aware Calibrated Efficient Routing} (RACER), a post-hoc and model-agnostic paradigm to solve the $\alpha$-VOR problem (Eq.~\eqref{eq:problem}).
RACER adaptively expands or contracts the model set based on router uncertainty. 
By employing a calibration procedure to determine a data-dependent threshold, it transforms raw router scores into calibrated set predictions with guaranteed risk control, all without retraining.

As illustrated in Figure~\ref{fig:paradigm}, RACER is composed of three key modules: (1) Augmented Scoring and Set Construction, which extends the scoring mechanism to handle abstention and parameterizes the routing decisions; (2) Risk Calibration, which optimizes a data-dependent threshold on a calibration dataset to control risk, and (3) Inference and Response Aggregation, which applies the calibrated threshold to new queries and aggregates the final output.

\subsection{Augmented scoring and set construction}

For any query $\bm{x} \in \mathcal{X}$, standard routers are ill-equipped to handle cases where $G(\bm{x}) = \emptyset$. To address this, we formalize the ``abstention" mechanism by introducing a virtual null model, denoted by $m_\emptyset$.
We define the \emph{augmented model pool} as $\mathcal{M}' = \mathcal{M} \cup \{m_\emptyset\}$. Accordingly, we map the original ground truth set $G(\bm{x})$ to an \emph{augmented ground truth set} $G'(\bm{x}) \subseteq \mathcal{M}'$ defined formally as follows:
\begin{equation} \label{eq:aug_ground_truth}
    G'(\bm{x})=
    \begin{cases}
        G(\bm{x}), & \text{if } G(\bm{x})\neq\emptyset,\\
        \{m_\emptyset\}, & \text{if } G(\bm{x})=\emptyset.
    \end{cases}
\end{equation}
This transformation ensures that $G'(\bm{x})$ is guaranteed to be never empty, explicitly treating the selection of $m_\emptyset$ as the correct decision when all candidate LLMs fail.

To perform routing over the augmented set $\mathcal{M}'$, we extend the base scoring function $f: \mathcal{X} \times \mathcal{M} \to \mathbb{R}$ to an \emph{augmented router score} $r: \mathcal{X} \times \mathcal{M}' \to \mathbb{R}$. We define a function $\phi: \mathbb{R}^{|\mathcal{M}|} \to \mathbb{R}$ to synthesize a score for the null model based on the confidence of standard models:
\begin{equation} \label{eq:aug_score}
	r(\bm{x}, m) =
	\begin{cases}
		f(\bm{x}, m) & \text{if } m \in \mathcal{M}, \\
		\phi\left( \{f(\bm{x}, k) : k \in \mathcal{M}\} \right) & \text{if } m = m_\emptyset.
	\end{cases}
\end{equation}
This formulation preserves the exchangeability of the resulting scores, as $m_\emptyset$'s score depends solely on the current query. 
The augmented router score $r(\bm{x}, m)$ is subsequently transformed into a non-conformity score $s(\bm{x}, m)$ via a monotonic mapping. The specific choice of $\phi$ and the non-conformity construction $s(\bm{x},m)$ are detailed in Section~\ref{subsec:setup}.

Based on non-conformity score $s(\bm{x}, m)$, we construct a parameterized family of functions $\{C_\lambda\}_\lambda$. For a given threshold $\lambda$, the prediction model set is defined as:
\begin{equation}\label{eq:prediction_set_def} 
    C_\lambda(\bm{x}) = \left\{ m \in \mathcal{M}' : s(\bm{x}, m) \le \lambda \right\}. 
\end{equation}
By varying $\lambda$, we obtain a nested sequence of sets, allowing us to tune the trade-off between the \textit{size} and \textit{risk}.

\subsection{Risk calibration}
The core objective of RACER is to determine a data-dependent threshold $\hat{\lambda}$ satisfying the $\alpha$-VOR constraint. We assume access to a finite labeled calibration dataset $\mathcal{D}_{\mathrm{cal}} = \{ (\bm{x}_i, G(\bm{x}_i)) \}_{i=1}^n$, where $\bm{x}_i$ is drawn exchangeably from the same distribution as the test data, $G(\bm{x}_i)\subseteq \mathcal{M}$ denotes the set of ground truth LLMs for $\bm{x}_i$. For clarity and notational convenience, throughout the remainder of this paper, we use $G_i$ to denote $G(\bm{x}_i)$.

\paragraph{Loss function.}
For each calibration sample, we compute the augmented ground truth set $G'_i$ via Eq.~\eqref{eq:aug_ground_truth} and the non-conformity scores $\{s(\bm{x}_i, m)\}_{m \in \mathcal{M}'}$. The misrouting loss for a given threshold $\lambda$ is defined as:
\begin{equation} \label{eq:loss}
    l(C_{\lambda}(\bm{x}_i), G'_i) = \mathbf{1}\bigl\{ C_\lambda(\bm{x}_i) \cap G'_i = \emptyset \bigr\}. 
\end{equation}
We denote this as $L_i(\lambda) = l(C_{\lambda}(\bm{x}_i), G'_i)$. By definition of $C_\lambda(\bm{x}_i)$ in Eq.~\eqref{eq:prediction_set_def}, the prediction model set fails to cover $G'_i$ if and only if $\lambda$ is strictly smaller than the minimum score among the ground truth LLMs. 
Let $s_i = \min_{m \in G'_i} s(\bm{x}_i, m)$ represent the critical non-conformity score for $\bm{x}_i$. Consequently, the loss can be equivalently expressed as 
$$L_i(\lambda) = \mathbf{1}\{ s_i > \lambda \}.$$

\paragraph{Threshold selection.}
To guarantee risk control on unseen data, we leverage finite-sample concentration bounds. We define the empirical risk as $\bar{L}_n(\lambda) = \frac{1}{n} \sum_{i=1}^n L_i(\lambda)$. The optimal calibrated threshold $\hat{\lambda}$ is determined as:
\begin{equation}
	\hat{\lambda} = \inf\left\{ \lambda \in \Lambda : \frac{n}{n+1} \bar{L}_n(\lambda) + \frac{1}{n+1} \le \alpha \right\}.
	\label{eq:lambda_hat}
\end{equation}
Through this calibration procedure, RACER guarantees that on unseen data from the same distribution, 
$$\mathbb{E}[l(C_{\hat{\lambda}}(\bm{X}), G'(\bm{X}))] \le \alpha.$$

\subsection{Inference and response aggregation}
\label{subsec:aggregation}

Given a query \(\bm{x}_{n+1}\), we construct the prediction set \(C_{\hat\lambda}(\bm{x}_{n+1})\) based on Eq.~\eqref{eq:prediction_set_def} and (\ref{eq:lambda_hat}). 
If $ C_{\hat{\lambda}}(\bm{x}_{n+1}) \cap \mathcal{M} = \emptyset $, the system triggers abstention. 
Otherwise, we collect outputs $\{y_m\}$ from models $m \in C_{\hat\lambda}(\bm{x}_{n+1}) \cap \mathcal{M}$ and apply aggregation strategies to derive the final answer. We employ two aggregation methods: majority voting \cite{wang2022self} and weighted aggregation \cite{taubenfeld2025confidence}. 

\paragraph{Majority voting.}
We perform majority voting over the valid models $ C_{\hat{\lambda}}(\bm{x}_{n+1}) \cap \mathcal{M}$, using average router scores for tie-breaking. Let $M_y = \{ m \in C_{\hat{\lambda}}(\bm{x}_{n+1}) \cap \mathcal{M} : y_m = y \}$, and $ Y^* $ denote the set of answers with the maximum vote count $ |M_y| $. The final prediction $ \hat{y} $ is determined by:
\begin{equation*}
\label{eq:agg_vote_tiebreak} 
\hat{y} = \underset{y \in Y^*}{\arg\max} \frac{1}{|M_y|} \sum_{m \in M_y} r(\bm{x}_{n+1}, m). 
\end{equation*}
Note that if $ |Y^*| = 1 $, this reduces to standard majority voting; otherwise, it selects the candidate with the highest average router confidence among the tied answers.

\paragraph{Weighted aggregation.}
% incorporates model-specific confidence measures to 
In contrast, the weighted aggregation method assigns model-specific weights, normalized via a tunable temperature parameter $T$:
\begin{equation*}
	\tilde{w}_m = \frac{\exp(w_m / T)}{\sum_{m' \in C_{\hat{\lambda}}(\bm{x}_{n+1})} \exp(w_{m'} / T)}.
\end{equation*}
The final prediction maximizes the weighted sum of votes:
\begin{equation*}
\label{eq:agg_weighted}
	\hat{y} = \arg\max_{y} \sum_{m \in C_{\hat{\lambda}}(\bm{x}_{n+1})} \tilde{w}_m \mathbf{1}(a_m = y),
\end{equation*}
We evaluate three distinct weighting schemes for $w_m$: \textit{Base router scores}, \textit{Verbal binary confidence} \cite{lin2022teaching}, and \textit{$\bm P(\text{True})$ confidence} \cite{kadavath2022language}. 
The temperature $T$ is tuned independently for each model and weighting scheme using a $10\%$ validation set.
Detailed definitions of these metrics, as well as the comprehensive protocol for parameter selection are provided in Appendix~\ref{app:implementation_details}.

Notably, RACER offers several compelling advantages:
\begin{itemize}[itemsep=3pt]
    \item \textbf{Easy-to-use}: RACER is a post-hoc paradigm that requires no retraining of the base router or LLMs, allowing for lightweight and flexible deployment.
    \item \textbf{Model-agnostic}: It is compatible with arbitrary base routers and non-conformity score functions, universally enhancing them without architectural constraints.
    \item \textbf{Reliable}: RACER achieves rigorous distribution-free risk control at a user-specified level $\alpha$.
\end{itemize}

We summarize the complete RACER paradigm, including aggregation, in Algorithm \ref{alg:racer}. In the next section, we provide a detailed theoretical analysis demonstrating that RACER satisfies the validity constraint of the $\alpha$-VOR problem.

\begin{algorithm}[t]
	\caption{RACER}
	\label{alg:racer}
	
	\textbf{Input:} 
    Calibration data $\mathcal{D}_{\mathrm{cal}}=\{(\bm{x}_i, G_i)\}_{i=1}^n$, 
    validation data $\mathcal{D}_{\mathrm{val}}$,
    test query $\bm{x}$,
    base scoring function $f: \mathcal{X} \times \mathcal{M} \to \mathbb{R}$,
    risk level $\alpha$,
    LLM pool $\mathcal{M}$
	
	\textbf{Output:} Prediction set $C(\bm{x})$, Final answer $\hat{y}$
	
	\begin{algorithmic}[1]
		\STATE Initialize augmented model pool $\mathcal{M}' \leftarrow \mathcal{M} \cup \{m_\emptyset\}$.
		\STATE Initialize score set $S \leftarrow \emptyset$.
		\FOR{each $(\bm{x}_i, G_i) \in \mathcal{D}_{\mathrm{cal}}$}
            \STATE Compute $\{s(\bm{x}_i, m)\}_{m \in \mathcal{M}'}$ by $\{f(\bm{x}_i, m)\}_{m \in \mathcal{M}}$.
            \STATE $G'_i \leftarrow G_i$ \textbf{if} $G_i \neq \emptyset$ \textbf{else} $\{m_\emptyset\}$.
            % \STATE Compute non-conformity scores $\{s(\bm{x}_i, m)\}_{m \in \mathcal{M}'}$.
            \STATE $s_i \leftarrow \min_{m \in G'_i} s(\bm{x}_i, m)$.
            \STATE $S \leftarrow S \cup \{s_i\}$.
		\ENDFOR
        \STATE $\hat{\lambda} \leftarrow \inf\left\{ \lambda \in \Lambda : \frac{1 + \sum_{s_i \in S} \mathbf{1}\{ s_i > \lambda \}}{n+1} \le \alpha \right\}$
        \STATE Compute aggregation parameters $\bm{\theta}$ using $\mathcal{D}_{\mathrm{val}}$.
        \STATE Compute $\{s(\bm{x}, m)\}_{m \in \mathcal{M}'}$ given $\{f(\bm{x}, m)\}_{m \in \mathcal{M}}$.
		\STATE $C_{\hat{\lambda}}(\bm{x}) \leftarrow \{ m \in \mathcal{M}' : s(\bm{x}, m) \le \hat{\lambda} \}$.
		
		\IF{$C_{\hat{\lambda}}(\bm{x}) \cap \mathcal{M} = \emptyset$}
		    \STATE $\hat{y} \leftarrow \texttt{Abstain}$
		\ELSE
            \STATE $\mathcal{M}_{\text{sel}} \leftarrow C_{\hat{\lambda}}(\bm{x}) \cap \mathcal{M}$.
		    \STATE $\mathcal{Y}_{\text{out}} \leftarrow \{ m(\bm{x}) : m \in \mathcal{M}_{\text{sel}} \}$.
		    \STATE $\hat{y} \leftarrow \text{Aggregate}(\mathcal{Y}_{\text{out}}; \bm{\theta})$.
		\ENDIF
		
		\STATE \textbf{Return} $C_{\hat{\lambda}}(\bm{x})$, $\hat{y}$
	\end{algorithmic}
\end{algorithm}

\section{Theoretical analysis}
In this section, we provide a rigorous theoretical analysis of the RACER paradigm.
We first establish fundamental structural properties, demonstrating that the constructed prediction sets form a nested family (Lemma~\ref{lem:nested}) and that the loss function is monotone, right-continuous, and bounded with respect to the threshold (Lemma~\ref{lem:loss}).
Finally, we present our main theoretical results: we prove that the calibrated threshold guarantees risk control on unseen queries at the user-specified level (Theorem~\ref{thm:risk_control}), and establish a matching risk lower bound (Theorem~\ref{thm:lower_bound}).

\subsection{Nestedness and monotonicity}
% \paragraph{Notation and definitions.}
Let $(\bm{X}, G'(\bm{X}))$ denote a random query and its augmented ground-truth set.
Recall that for a given threshold $\lambda\in\mathbb{R}$, the router-based function outputs $C_\lambda(\bm{X})\subseteq \mathcal{M}'$ and incurs the misrouting loss $l(C_\lambda(\bm{X}), G'(\bm{X}))$ as defined in Eq.~\eqref{eq:loss}.
The Risk of threshold $\lambda$ is given by:
\[
R(\lambda) := \mathbb{E}\!\left[l(C_\lambda(\bm{X}), G'(\bm{X}))\right].
\]
% which is the probability that the router excludes all optimal models in the augmented space.
Correspondingly, we define the \emph{coverage} as
$
\mathrm{Cov}(\lambda) := \mathbb{P}\!\left(C_\lambda(\bm{X}) \cap G'(\bm{X}) \neq \emptyset\right).
$
To ensure that the calibration procedure is well-posed, we analyze the structural properties of both the prediction model sets and the associated misrouting loss function in this section.

% 嵌套性
\begin{lemma}[Nestedness]
\label{lem:nested}
For any query $\bm{x}\in\mathcal{X}$, the prediction model sets
$\{C_\lambda(\bm{x})\}_{\lambda\in\mathbb{R}}$ defined in
Eq.~\eqref{eq:prediction_set_def} form a nested family. That is, for any
$\lambda_1\le \lambda_2$,
\[
C_{\lambda_1}(\bm{x}) \subseteq C_{\lambda_2}(\bm{x}).
\]
\end{lemma} 
The complete proof is given in Appendix~\ref{app:proof_lem_nested}.

% 连续和有界性
\begin{lemma}[Monotonicity, right-continuity, and boundedness]
    For any $(\bm{x},G'(\bm{x}))$, the loss $l(C_{\lambda}(\bm{x}),G'(\bm{x}))$ is a non-increasing and right-continuous function of $\lambda$. Moreover, for random $(\bm{X},G'(\bm{X}))$, we have $0\le l(C_{\lambda}(\bm{X}),G'(\bm{X}))\le 1$ almost surely.
\label{lem:loss}
\end{lemma}
The complete proof is given in Appendix~\ref{app:proof_lem_loss}.
Lemma~\ref{lem:nested} and Lemma~\ref{lem:loss} imply that the empirical risk is non-increasing in $\lambda$ and bounded, so the calibration problem is well-posed and can be solved via a one-dimensional search. When the score distribution has no ties (e.g., under a continuity assumption), the calibrated threshold is unique.

\subsection{Risk control}
We now present the main theoretical guarantee of the RACER paradigm. We show that the threshold $\hat{\lambda}$, selected via the calibration procedure on a finite calibration set, generalizes to control the expected loss on unseen data.
% alpha valid
\begin{theorem}[Risk control]
\label{thm:risk_control}
	Assume the augmented calibration data and the new query $\bm{X}_{n+1}$ are exchangeable. Let $\hat{\lambda}$ be the threshold returned by RACER (Eq.~\eqref{eq:lambda_hat}) with $\alpha\in(0,1)$. Then the \textit{Risk} of $\hat{\lambda}$ satisfies:
    \begin{equation*}\label{eq:risk_control_main}
        R(\hat{\lambda}) :=\mathbb{E}\bigl[l(C_{\hat\lambda}(\bm{X}_{n+1}),G'(\bm{X}_{n+1}))\bigr] \le \alpha,
    \end{equation*}
    where the expectation is taken over the calibration data as well as the new query $\bm{X}_{n+1}$.
\end{theorem}
The proof is provided in Appendix~\ref{app:proof_th_risk_control}. Theorem~\ref{thm:risk_control} guarantees that RACER controls the risk at the user-specified level $\alpha$. This property holds for any finite sample size $n$ and is distribution-free (assuming exchangeability). Consequently, the router remains reliable even in safety-critical settings where failure rates must be strictly bounded.

\begin{remark}
    A direct consequence of Theorem~\ref{thm:risk_control} is a guarantee on the coverage of the selected model sets. Since $\mathrm{Cov}(\lambda)=1-R(\lambda)$, Theorem~\ref{thm:risk_control} immediately implies
$\mathrm{Cov}(\hat{\lambda})\ge 1-\alpha$.
\end{remark}

\begin{figure*}[!t]
	\centering
	\begin{subfigure}{0.45\linewidth}
		\centering
		\includegraphics[width=\linewidth]{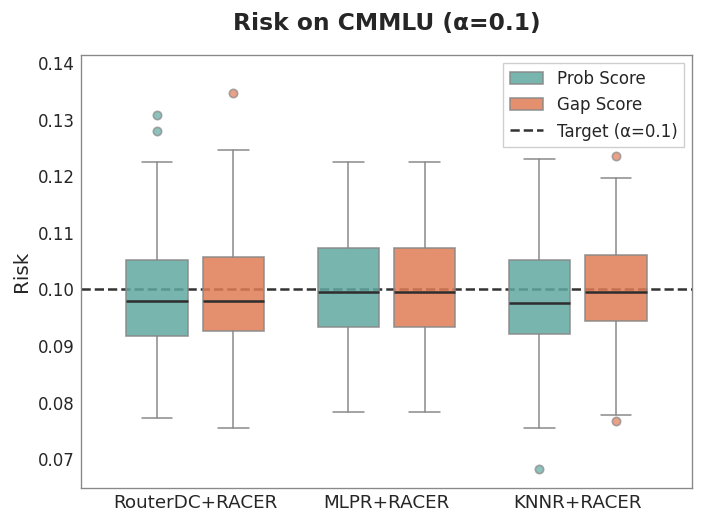}
	\end{subfigure}
	% \vspace{0.5em} % 控制上下两图的间距，可微调
	\begin{subfigure}{0.45\linewidth}
		\centering
		\includegraphics[width=\linewidth]{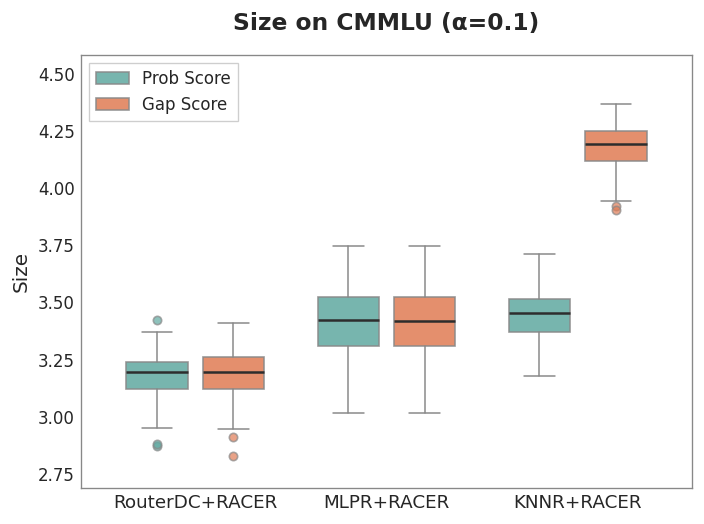}
	\end{subfigure}
	\caption{\textbf{Distributions of \textit{risk} and \textit{size} for RACER on CMMLU over 100 independent trials with a target risk level $\alpha=0.1$.} \textbf{Left}: The distribution of \textit{risk}, where the black dashed line represents the user-specified risk level. Results demonstrate that RACER consistently maintains the risk below the target $\alpha$ for all base routers and non-conformity scores. \textbf{Right}: The distribution of prediction set \textit{size}. The \textcolor[rgb]{0.4,0.7,0.6}{green} and \textcolor[rgb]{0.9,0.6,0.5}{orange} boxes represent the \textit{router score-gap} and \textit{inverse irobability}  non-conformity scores, respectively.}
 \label{fig:example_risk_size}
\end{figure*}

\begin{theorem}[Risk lower bound]
\label{thm:lower_bound}
% 修改表述
In the setting of RACER, assume that the calibration data and the test query are i.i.d., and the non-conformity score $s(\bm{X}, m)$ follows a continuous distribution (i.e., $\mathbb{P}(s(\bm{X}, m) = \lambda) = 0$ for any constant $\lambda$).
Let $\hat{\lambda}$ be the threshold calibrated at level $\alpha \in (0,1)$. Then, the {Risk} on a new query $\bm{X}_{n+1}$ is lower-bounded by:
\begin{equation*}
    % \mathbb{E}\bigl[l(C_{\hat\lambda}(\bm{X}_{n+1}),G'(\bm{X}_{n+1}))\bigr] \ge \alpha - \frac{2}{n+1}.
    R(\hat{\lambda}) \ge \alpha - \frac{2}{n+1}.
\end{equation*}
\end{theorem}

We prove it in Appendix~\ref{app:proof_th_risk_lower}. 
While Theorem~\ref{thm:risk_control} guarantees validity by upper-bounding the risk,
Theorem~\ref{thm:lower_bound} complements it from the other side and shows that RACER is not overly conservative. Together, these results imply that the calibrated procedure achieves a near-tight risk guarantee:
\[
\alpha-\frac{2}{n+1}\ \le\ R(\hat{\lambda})\ \le\ \alpha.
\]
Hence, the achieved risk is within $O(1/n)$ of the target level $\alpha$. As the calibration size $n$ increases, this gap vanishes, indicating that RACER approaches the target risk level.

\section{Experiments}
In this section, we present the experimental results to validate two key hypotheses: 
(i) RACER achieves rigorous risk control;
(ii) The final aggregated output can improve downstream performance over base routers and full models.
We validate these hypotheses across multiple benchmarks and three base routers. 
% The experimental setup is provided in Section \ref{subsec:setup}.
% Results on risk control and aggregation accuracy are discussed thoroughly in Section \ref{subsec:result}.

\begin{table*}[t]
\centering
\caption{\textbf{Performance comparison on diverse benchmarks.} We report the \textbf{mean and standard deviation ($\pm$ std)} of accuracy across \textbf{100 independent trials}. The ``Average'' column represents the average accuracy across all four benchmarks. ``Base'' refers to the base routers, while ``+ RACER'' indicates our proposed paradigm with aggregation (w/ Agg.). RACER-G and RACER-P denote two RACER variants with different non-conformity score definitions, based on the \textit{router score-gap} and \textit{inverse irobability}, respectively. The best results in each column are \textbf{bolded}, and the second best are \underline{underlined}. }
\label{tab:main_accuracy}
\renewcommand\arraystretch{1.0}
\resizebox{1.00\textwidth}{!}{%
\setlength{\tabcolsep}{6mm}{
% \small
\footnotesize
    \begin{tabular}{lcccc|c}
        \toprule
        \textbf{Method} & \textbf{GSM8K} & \textbf{MMLU} & \textbf{C-MMLU} & \textbf{ARC-C} & \textbf{Average} \\
        \midrule
        \multicolumn{6}{l}{\textit{Candidate Models}} \\
        \midrule
        \rowcolor{gray!10} Chinese-Mistral-7B-v0.1 & 42.7$\pm$1.4 & 57.5$\pm$0.8 & 49.1$\pm$1.0 & 45.3$\pm$2.4 & 48.7 \\
        \rowcolor{gray!10} Dolphin-2.6-Mistral-7b & 54.7$\pm$1.5 & 60.7$\pm$0.8 & 44.2$\pm$0.9 & 53.7$\pm$2.1 & 53.3 \\
        \rowcolor{gray!10} Dolphin-2.9-Llama3-8b & 75.0$\pm$1.2 & 59.6$\pm$0.8 & 43.9$\pm$0.9 & 49.8$\pm$2.2 & 57.1 \\
        \rowcolor{gray!10} Meta-Llama-3-8B & 47.1$\pm$1.5 & 64.6$\pm$0.8 & 51.2$\pm$0.9 & 50.5$\pm$2.2 & 53.4 \\
        \rowcolor{gray!10} MetaMath-Mistral-7B & 75.1$\pm$1.2 & 59.9$\pm$0.8 & 44.4$\pm$1.0 & 48.7$\pm$2.3 & 57.0 \\
        \rowcolor{gray!10} Mistral-7B-v0.1 & 37.5$\pm$1.3 & 62.1$\pm$0.7 & 44.6$\pm$1.0 & 50.7$\pm$2.1 & 48.7 \\
        \rowcolor{gray!10} Zephyr-7b-beta & 34.0$\pm$1.3 & 59.6$\pm$0.8 & 43.5$\pm$1.0 & 57.6$\pm$2.1 & 48.7 \\
        \midrule
        \multicolumn{6}{l}{\textit{Routers \& RACER}} \\
        \midrule
        MLPR & 75.1$\pm$1.2 & 59.9$\pm$0.8 & 44.4$\pm$1.0 & 48.7$\pm$2.3 & 57.0 \\
        + RACER-G (w/ Agg.) & 76.9$\pm$1.4 & 63.3$\pm$0.9 & 48.3$\pm$1.2 & 52.5$\pm$2.7 & 60.3 \\
        + RACER-P (w/ Agg.) & \underline{77.8}$\pm$1.2 & \textbf{63.7}$\pm$1.0 & 48.4$\pm$1.2 & 52.3$\pm$2.6 & 60.6 \\
        \midrule
        KNNR & 74.1$\pm$1.2 & 62.5$\pm$0.8 & 48.1$\pm$0.9 & 54.6$\pm$2.2 & 59.8 \\
        + RACER-G (w/ Agg.) & 76.6$\pm$1.1 & \underline{63.4}$\pm$0.8 & 48.8$\pm$0.9 & 54.5$\pm$2.3 & 60.8 \\
        + RACER-P (w/ Agg.) & 77.3$\pm$1.6 & \underline{63.4}$\pm$1.0 & 49.0$\pm$0.9 & 55.0$\pm$2.2 & 61.2 \\
        \midrule
        RouterDC & 75.0$\pm$1.2 & 61.0$\pm$0.8 & \textbf{51.0}$\pm$0.9 & \underline{56.7}$\pm$2.3 & 60.9 \\
        + RACER-G (w/ Agg.) & 77.4$\pm$1.2 & 62.9$\pm$1.0 & \underline{50.7}$\pm$1.0 & 56.4$\pm$2.4 & \underline{61.9} \\
        + RACER-P (w/ Agg.) & \textbf{77.9}$\pm$1.2 & 63.0$\pm$0.9 & 50.6$\pm$1.0 & \textbf{56.8}$\pm$2.5 & \textbf{62.1} \\
        \bottomrule
    \end{tabular}%
}
}
\end{table*}

\begin{figure*}
    \centering
    \includegraphics[width=0.8\linewidth]{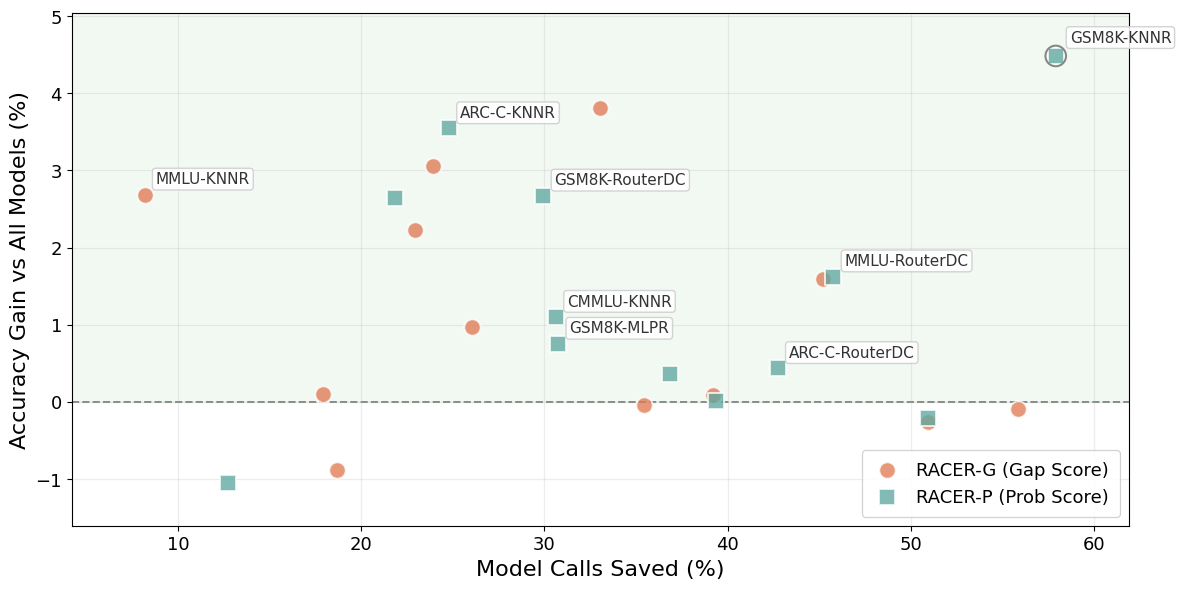}
    \caption{\textbf{Trade-off between computational efficiency and performance.} The scatter plot illustrates the reduction in inference overhead (Model Calls Saved) versus the absolute improvement in test accuracy (Accuracy Gain) for RACER relative to the full model ensemble. The concentration of data points in the upper-right region indicates that RACER effectively filters noise to achieve significant savings (up to 58.6\%) while simultaneously improving accuracy (up to +4.49\%) across various dataset-router configurations.}
    \label{fig:model_calls_save}
\end{figure*}

\subsection{Setup}
\label{subsec:setup}

\paragraph{Datasets.}
We evaluate RACER across four diverse benchmarks:\textbf{MMLU} \cite{hendrycks2021measuring}, covering 57 subjects in general knowledge; \textbf{GSM8K} \cite{cobbe2021training}, focused on grade-school mathematics; \textbf{CMMLU} \cite{li2024cmmlu}, a comprehensive Chinese benchmark spanning 67 subjects; and \textbf{ARC-Challenge} \cite{clark2018think}, designed for reasoning tasks.
We construct a unified training set for base routers and partition the remaining data into calibration, validation, and test sets (details in Appendix~\ref{app:implementation_details}).
% To train the base routers, we construct a unified global training set $\mathcal{D}_{\text{train}}$ by merging training splits from all benchmarks.
% The remaining data is partitioned into validation, calibration, and test sets for parameter selection, threshold calibration, and final evaluation, respectively. 
% We repeat this procedure over 100 random splits, with
% comprehensive details regarding dataset splitting and ground-truth label construction provided in Appendix~\ref{app:implementation_details}.

\paragraph{Candidate LLMs and base routers.}
We use a pool of seven open-source LLMs: (i) \textit{Mistral-7B} \cite{jiang2023mistral7b}; (ii) \textit{MetaMath-Mistral-7B} \cite{yu2024metamath}; (iii) \textit{Zephyr-7b-beta} \cite{tunstall2023zephyr}; (iv) \textit{Chinese-Mistral-7B} \cite{Chinese-Mistral}; (v) \textit{Dolphin-2.6-mistral-7b} \cite{dphn2024dolphinmistral}; (vi) \textit{Llama-3-8B} \cite{grattafiori2024llama}; (vii) \textit{Dolphin-2.9-llama3-8b} \cite{hartford2024dolphin}. 
% The first five LLMs are Mistral-based, while the last two LLMs are Llama-3-based.
We choose the following base routers: RouterDC~\cite{chen2024routerdc},  MLPR~\cite{huang2025routereval} and KNNR~\cite{hu2024routerbench}. 
Detailed introduction, hyperparameter settings, and configurations for each base router are provided in Appendix~\ref{app:implementation_details}. 

% \paragraph{Baselines.}

% \begin{itemize}[itemsep=3pt, parsep=0pt, topsep=1pt]
% 	\item RouterDC \cite{chen2024routerdc}. RouterDC consists of an encoder and LLM embeddings, utilizing dual contrastive learning to train the router.
% 	\item MLPR \cite{huang2025routereval}. A multi-layer perceptron (MLP) classifier is trained as the router to predict model performance score.
% 	\item KNNR \cite{hu2024routerbench}. A k-nearest neighbors (KNN) classifier is trained as the router, using the classification probabilities of each LLM as the router score.
% \end{itemize} 
% After training, we apply our RACER paradigm as a post-processing step on the trained routers. 

\paragraph{Augmented scoring and non-conformity scores.}
To calculate the augmented router score in Eq.~\eqref{eq:aug_score}, we design $\phi(\cdot)$ to quantify \textit{system uncertainty}, which increases when base scores are uniform (i.e., the router is ambiguous). We adopt a \emph{max-based} strategy: $r(\bm{x}, m_\emptyset) = 1 - \max_{k \in \mathcal{M}} f(\bm{x}, k)$, effectively captures the absence of a dominant candidate LLM. 
Based on these augmented scores, we implement Algorithm \ref{alg:racer} using two distinct non-conformity scores: \textit{Router Score-Gap} and \textit{Inverse Probability}. We apply randomized smoothing to resolve numerical ties. Detailed definitions of these scores are provided in Appendix~\ref{app:implementation_details}.
\paragraph{Metrics.}
We assess performance using three metrics: \textit{Risk} (the misrouting risk of the set), \textit{Size} (the average number of candidate LLMs in the set), and \textit{Accuracy} (the final aggregated performance).
Formal definitions are detailed in Appendix~\ref{app:implementation_details}.

% To evaluate the effectiveness of the RACER paradigm, we employ three key metrics: average \textit{Risk}, \textit{Size} and \textit{Accuracy}. These metrics are formally defined as:
% \begin{equation*}
% 	\begin{aligned}
% 		& \text{Risk} := \frac{1}{N}\sum_{i=1}^{N} \mathbf{1}(C_{\hat \lambda}(\bm{x}_i)\cap G'_i = \emptyset), \\
%         & \text{Size} := \frac{1}{N}\sum_{i=1}^{N} \left| C_{\hat \lambda}(\bm{x}_i) \cap \mathcal{M} \right|,\\
% 		& \text{Accuracy} := \frac{1}{N}\sum_{i=1}^{N}\mathbf{1}(\hat{y}_i =y^*_i ).
% 	\end{aligned}
% \end{equation*}
% where $N$ is the number of test queries, $G'_i$ denotes the augmented ground truth set, $\hat{y}_i$ and $y^*_i$ represent the final aggregated output of RACER and the correct answer, respectively. 
% Crucially, the definition of \textit{size} uses the intersection with $\mathcal{M}$ to explicitly exclude the virtual null model $m_\emptyset$. This ensures the metric faithfully reflects the actual inference overhead by counting only the candidate LLMs.

% \paragraph{Implementation Details.}

\subsection{Results}
\label{subsec:result}

\paragraph{RACER strictly controls the risk of excluding all optimal LLMs.}
Figure~\ref{fig:example_risk_size} presents the distributions of empirical risk and prediction set size on CMMLU ($\alpha=0.1$).
First, RACER achieves rigorous risk control: across all base routers and non-conformity measures, the empirical risk consistently satisfies the theoretical bound of the target level $\alpha$.
Second, the average size of prediction sets effectively reflects the capability of the underlying base router; RouterDC yields the most compact sets compared to MLPR and KNNR, indicating superior ranking performance. 
Finally, the \emph{inverse probability} non-conformity score consistently produces smaller sets than the router score-gap. Notably, the size gap between the two non-conformity measures is more pronounced on less capable base routers (e.g., KNNR).
%highlighting the superior efficiency of the \textit{inverse probability} score.
% These findings are consistent across all benchmarks, as detailed in the full results in Appendix~\ref{app:detail_risk_size}.

\paragraph{RACER significantly enhances downstream accuracy via aggregation.} 
We instantiate RACER with two scoring variants:
RACER-G (router score-gap) and RACER-P (inverse probability). 
In the following analysis, we use RACER to refer to the complete method incorporating aggregation strategies (w/ Agg.). Table~\ref{tab:main_accuracy} presents the performance comparison.
First, RACER consistently improves over base routers.
% Although RouterDC leaves limited room for improvement on certain tasks, RACER still outperforms it in overall average accuracy across all datasets.
Notably, RACER-P achieves a maximum accuracy gain of 4.0\% on a single benchmark (MLPR on C-MMLU) and a maximum average improvement of 3.6\% across all benchmarks (MLPR).
% , validating its ability to robustly enhance routing methods.
Second, RACER surpasses the best individual LLM on average.
In terms of overall average accuracy, RACER attains 62.1\%, exceeding the best individual LLM in the candidate model pool (57.1\%) by 5.0\%.
Overall, these results demonstrate that RACER effectively mitigates the limitations of single-model selection.
% and translates controlled risk into stronger downstream performance through aggregation. 

\paragraph{RACER achieves higher accuracy while reducing model calls against full models.}
We evaluate the benefits of RACER paradigm compared to the \textit{full model aggregation} baseline.
Figure~\ref{fig:model_calls_save} visualizes the trade-off between model calls reduction and accuracy gain. Ideally, a method should reside in the upper-right quadrant. The results demonstrate that RACER achieves a ``win-win'' outcome: it significantly reduces the number of model calls, saving up to \textbf{58.6\%} of model calls, while simultaneously boosting test accuracy by up to \textbf{4.49\%} over the full ensemble. This indicates that the models excluded by RACER are not merely redundant but often detrimental to the aggregation process, allowing for substantial resource savings without compromising downstream task performance.

\paragraph{Additional results.}
Due to space limitations, we present further analyses in the Appendix. 
While Figure~\ref{fig:example_risk_size} illustrates representative results, Appendix~\ref{app:detail_risk_size} provides comprehensive distributional analyses of risk and size at $\alpha=0.1$ across all experimental settings.
Appendix~\ref{app:detail_acc} visualizes the corresponding optimal aggregation hyperparameters used in Table~\ref{tab:main_accuracy}. 
Appendix~\ref{app:ex_multiple_alpha} investigates the evolution of RACER's risk control and efficiency as $\alpha$ varies. 
% Finally, Appendix~\ref{app:vs_full_model} compares RACER against the full model aggregation, demonstrating that RACER achieves higher accuracy while reducing model calls by up to $58.6\%$.
Finally, Appendix~\ref{app:vs_full_model} supplements the accuracy comparison between RACER and full model aggregation.

\section{Related work}

Our work intersects with two key areas of research: efficient model routing in multi-LLM systems and conformal prediction for reliable risk control. 

\paragraph{LLM routing.}
The emergence of LLMs has prompted extensive research efforts aimed at optimizing the cost-performance trade-off in multi-model systems \cite{chen2025harnessing}. Early approaches focused on cascading strategies with high computational cost \cite{chen2024frugalgpt}. Predictive routers reduce cost by directly selecting models without calling for all LLMs \cite{hu2024routerbench, shnitzer2023large, srivatsa2024harnessing}. To enhance accuracy, various architectures leveraging contrastive learning \cite{chen2024routerdc}, graph networks \cite{feng2025graphrouter}, or preference-based ranking \cite{jiang2023llm, ong2025routellm} have been proposed. 
More recently, adaptive paradigms have emerged to optimize test-time compute via thresholds \cite{ding2024hybrid, ding2025bestroute} or reinforcement learning \cite{zhang2025router}. 
Our work introduce a universal, post-hoc paradigm designed to seamlessly enhance any of these deterministic routers with rigorous risk control capabilities.

\paragraph{Predictive inference.}
Predictive inference aims to quantify uncertainty and provide rigorous statistical guarantees for model outputs, exemplified by frameworks like Conformal Prediction (CP) \cite{vovk2005algorithmic,angelopoulos2021gentle} and Conformal Risk Control (CRC) \cite{angelopoulos2024conformal}. While CP constructs valid prediction sets independent of the data distribution, CRC generalizes this to control the expected value of arbitrary monotone loss functions.
In the context of LLMs, these techniques have been adapted to ensure generation validity \cite{quach2024conformal,kumar2023conformal} and control performance loss in efficient reasoning \cite{zeng2025note,zeng2025pac}.
Most relevant to our work, CP-Router \cite{su2025cp} utilizes uncertainty estimates for binary routing between an LLM and an LRM.
However, unlike CP-Router's specific switching mechanism, RACER applies calibration strategy directly to the \textit{router's decision space}, proposing a universal paradigm that constructs calibrated model sets to strictly satisfy the $\alpha$-VOR constraint across arbitrary model pools.

\section{Conclusion}

In this paper, we presented RACER, a novel post-hoc and model-agnostic paradigm designed to optimize the cost-performance trade-off in multi-model systems. By formulating the routing task as the $\alpha$-Valid Optimal Routing ($\alpha$-VOR) problem, RACER transforms standard single-model selection into calibrated set predictions. 
% This paradigm also supports abstention when no suitable candidate is identified.
Theoretically, we established rigorous, distribution-free guarantees, ensuring that the risk is controlled below a user-specified level, while also providing a risk lower bound to verify non-conservativeness.
Empirically, RACER maintains precise risk control across diverse benchmarks while significantly enhancing downstream accuracy via aggregation, outperforming both base routers and the best individual LLMs.
% This paradigm generalizes naturally to various large generative models as long as a routing mechanism exists.
As a plug-and-play solution applicable to diverse generative models utilizing arbitrary scoring functions, RACER grounds multi-model deployment in a solid statistical framework, and we hope this work facilitates future research into risk-aware routing for complex agentic workflows.

\section*{Impact Statement}
This paper presents work whose goal is to advance the field of machine learning. There are many potential societal consequences of our work, none of which we feel must be specifically highlighted here.

% In the unusual situation where you want a paper to appear in the references without citing it in the main text, use \nocite

\bibliography{ref}
\bibliographystyle{style/icml2026}

%%%%%%%%%%%%%%%%%%%%%%%%%%%%%%%%%%%%%%%%%%%%%%%%%%%%%%%%%%%%%%%%%%%%%%%%%%%%%%%
%%%%%%%%%%%%%%%%%%%%%%%%%%%%%%%%%%%%%%%%%%%%%%%%%%%%%%%%%%%%%%%%%%%%%%%%%%%%%%%
% APPENDIX
%%%%%%%%%%%%%%%%%%%%%%%%%%%%%%%%%%%%%%%%%%%%%%%%%%%%%%%%%%%%%%%%%%%%%%%%%%%%%%%
%%%%%%%%%%%%%%%%%%%%%%%%%%%%%%%%%%%%%%%%%%%%%%%%%%%%%%%%%%%%%%%%%%%%%%%%%%%%%%%
\newpage
\appendix
\onecolumn

\section{Proofs of theoretical results}
\subsection{Proof of Lemma~\ref{lem:nested}}
\label{app:proof_lem_nested}

\begin{proof}
For a query $\bm{x}\in\mathcal{X}$ and thresholds
$\lambda_1,\lambda_2\in\mathbb{R}$ such that $\lambda_1\le \lambda_2$.
Recall the definition of the routing model set in Eq.~\eqref{eq:prediction_set_def}:
\[
C_\lambda(\bm{x})
= \left\{ m \in \mathcal{M}' : s(\bm{x}, m) \le \lambda \right\}.
\]

To prove $C_{\lambda_1}(\bm{x}) \subseteq C_{\lambda_2}(\bm{x})$,
take any model $m\in C_{\lambda_1}(\bm{x})$.
By the definition of $C_{\lambda_1}(\bm{x})$, we have
\[
s(\bm{x},m) \le \lambda_1.
\]
Since $\lambda_1\le \lambda_2$, it follows that
\[
s(\bm{x},m) \le \lambda_1 \le \lambda_2,
\]
which implies $s(\bm{x},m)\le \lambda_2$.
Therefore, by the definition of $C_{\lambda_2}(\bm{x})$, we conclude that
$m\in C_{\lambda_2}(\bm{x})$.

Since the choice of $m\in C_{\lambda_1}(\bm{x})$ was arbitrary, we conclude that
\[
C_{\lambda_1}(\bm{x}) \subseteq C_{\lambda_2}(\bm{x}).
\]
This establishes the nestedness property and completes the proof of Lemma~\ref{lem:nested}.
\end{proof}

\subsection{Proof of Lemma~\ref{lem:loss}}
\label{app:proof_lem_loss}

\begin{proof}
Fix an arbitrary pair $(\bm{x},G'(\bm{x}))$ and recall the RACER loss in
Eq.~\eqref{eq:loss}:
\[
l(\bm{x}, G'(\bm{x}); \lambda)
= \mathbf{1}\Bigl\{ C_\lambda(\bm{x}) \cap G'(\bm{x}) = \emptyset \Bigr\}.
\]

\paragraph{Step 1: Monotonicity in $\lambda$.}
Let $\lambda_1\le \lambda_2$.
By Lemma~\ref{lem:nested}, we have
$C_{\lambda_1}(\bm{x})\subseteq C_{\lambda_2}(\bm{x})$.
Intersecting both sides with $G'(\bm{x})$ yields
\[
C_{\lambda_1}(\bm{x})\cap G'(\bm{x})
\subseteq
C_{\lambda_2}(\bm{x})\cap G'(\bm{x}).
\]
Consequently, the event that the larger intersection is empty implies the
smaller one is also empty, i.e.,
\[
\Bigl\{C_{\lambda_2}(\bm{x})\cap G'(\bm{x})=\emptyset\Bigr\}
\subseteq
\Bigl\{C_{\lambda_1}(\bm{x})\cap G'(\bm{x})=\emptyset\Bigr\}.
\]
Taking indicator functions on both sides gives
\[
l(\bm{x},G'(\bm{x});\lambda_2)
\le
l(\bm{x},G'(\bm{x});\lambda_1),
\]
which proves that $l(\bm{x},G'(\bm{x});\lambda)$ is non-increasing in $\lambda$.

\paragraph{Step 2: Right-continuity.}

Recall that the routing model set is defined as $C_\lambda(\bm{x}) = \{ m \in \mathcal{M}' : s(\bm{x}, m) \le \lambda \}$.
The loss function $l(\bm{x}, G'(\bm{x}); \lambda)$ takes the value $0$ if and only if $C_\lambda(\bm{x}) \cap G'(\bm{x}) \neq \emptyset$, which means there exists at least one valid model $m \in G'(\bm{x})$ such that $s(\bm{x}, m) \le \lambda$.

Let $S_{\min}(\bm{x})$ denote the minimum non-conformity score among the valid models:$$S_{\min}(\bm{x}) := \min_{m \in G'(\bm{x})} s(\bm{x}, m).$$Since $G'(\bm{x})$ is non-empty (it contains at least the null model) and finite, $S_{\min}(\bm{x})$ is well-defined.The condition for zero loss can thus be rewritten as $\lambda \ge S_{\min}(\bm{x})$. Conversely, the loss is $1$ if and only if $\lambda < S_{\min}(\bm{x})$.We can explicitly write the loss function as a step function:$$l(\bm{x}, G'(\bm{x}); \lambda) =
\begin{cases}
1, & \text{if } \lambda < S_{\min}(\bm{x}), \\
0, & \text{if } \lambda \ge S_{\min}(\bm{x}).
\end{cases}$$

The function is piecewise constant and the only point of discontinuity is at $\lambda^* = S_{\min}(\bm{x})$. To verify right-continuity at $\lambda^*$, we observe:
$$\lim_{\epsilon \downarrow 0} l(\bm{x}, G'(\bm{x}); \lambda^* + \epsilon) = 0,$$
since $\lambda^* + \epsilon > S_{\min}(\bm{x})$.
Also, by the definition of the case $\lambda \ge S_{\min}(\bm{x})$, the function value at the point is $l(\bm{x}, G'(\bm{x}); \lambda^*) = 0$.
Since the limit from the right equals the function value, $l(\bm{x}, G'(\bm{x}); \lambda)$ is right-continuous.

\paragraph{Step 3: Almost sure boundedness.}
For any $\lambda$, the loss is an indicator function and therefore takes values
in $\{0,1\}$. Hence, for random $(\bm{X},G'(\bm{X}))$,
\[
0 \le l(\bm{X},G'(\bm{X});\lambda) \le 1
\quad \text{almost surely}.
\]
This completes the proof of Lemma~\ref{lem:loss}.
\end{proof}

\subsection{Proof of Theorem \ref{thm:risk_control}}
\label{app:proof_th_risk_control}

\begin{proof}
Let $\mathcal{D}_{cal} = \{(\bm{X}_i, G'(\bm{X}_i))\}_{i=1}^n$ be the calibration set and $(\bm{X}_{n+1}, G'(\bm{X}_{n+1}))$ denotes the random, unseen data. We assume that the sequence $\{(\bm{X}_i, G'(\bm{X}_i))\}_{i=1}^{n+1}$ is exchangeable.

Recall the definition of the calibrated threshold $\hat{\lambda}$:
\begin{equation} 
\hat{\lambda} = \inf \left\{ \lambda \in \mathbb{R} : \frac{1}{n+1}\sum_{i=1}^n L_i(\lambda) + \frac{1}{n+1} \le \alpha \right\}. 
\label{eq:proof_lambda_def} 
\end{equation}
From Lemma \ref{lem:loss}, we know that the loss is non-increasing and bounded by 0 (the minimum possible loss is 0, and we assume $\alpha > 0$), this set is non-empty and $\hat{\lambda}'$ is well-defined.
Let $\hat{R}_{n+1}(\lambda) = \frac{1}{n+1} \sum_{i=1}^{n+1} L_i(\lambda)$.
We define the virtual threshold $\hat{\lambda}'$ as:
\begin{equation*}
    \hat{\lambda}' = \inf \left\{ \lambda \in \mathbb{R} : \hat{R}_{n+1}(\lambda) \le \alpha \right\}.
\end{equation*}
% Since the loss is non-increasing and bounded by 0, this set is non-empty and $\hat{\lambda}'$ is well-defined.
Since the loss $L_i(\lambda) = \mathbf{1}\{C_\lambda(\bm{X}_i) \cap G'(\bm{X}_i) = \emptyset\} \in [0, 1]$, we observe that for any $\lambda$:
\begin{equation*}
    \begin{aligned}
        \hat{R}_{n+1}(\lambda) = &\frac{1}{n+1}\sum_{i=1}^n L_i(\lambda) + \frac{L_{n+1}(\lambda)}{n+1}\\
        \leq & \frac{1}{n+1}\sum_{i=1}^n L_i(\lambda) + \frac{1}{n+1}.
    \end{aligned}
\end{equation*}

Since each $L_i(\lambda)$ is a non-increasing function of $\lambda$. 
If a specific $\lambda$ satisfies the condition for $\hat{\lambda}$ (i.e., LHS of Eq.~\eqref{eq:proof_lambda_def} $\le \alpha$), the inequality above implies that $\hat{R}_{n+1}(\lambda) \le \alpha$. Thus, $\hat{\lambda}' \le \hat{\lambda}$  almost surely, which implies $L_{n+1}(\hat{\lambda}) \le L_{n+1}(\hat{\lambda}')$. Taking the expectation:
\begin{equation} 
    \mathbb{E}[L_{n+1}(\hat{\lambda})] \le \mathbb{E}[L_{n+1}(\hat{\lambda}')]. \label{eq:proof_expect_ineq} 
\end{equation}

Let $E = \bigl\{(\bm{X}_i, G'(\bm{X}_i))\bigr\}_{i=1}^{n+1}$ be the unordered multiset of data points, the virtual threshold $\hat{\lambda}'$ depends only on the set $E$. Due to exchangeability, conditional on $E$,  $\hat{\lambda}'$ is fixed, and $\bm{X}_{n+1}$ is uniformly distributed over $E$.
Thus, the conditional expectation of the loss at $\hat{\lambda}'$ is:
\begin{equation*}
    \mathbb{E}[L_{n+1}(\hat{\lambda}') \mid E] = \frac{1}{n+1} \sum_{i=1}^{n+1} L_i(\hat{\lambda}') = \hat{R}_{n+1}(\hat{\lambda}').
\end{equation*}
By the definition of $\hat{\lambda}'$ and the right-continuity of the loss function (Lemma \ref{lem:loss}), we have $\hat{R}_{n+1}(\hat{\lambda}') \le \alpha$. Therefore $\mathbb{E}[L_{n+1}(\hat{\lambda}') \mid E] \le \alpha.$
Applying the law of total expectation to Eq.~\eqref{eq:proof_expect_ineq}:
\begin{equation*}
    \mathbb{E}[L_{n+1}(\hat{\lambda})] \le \mathbb{E}[\mathbb{E}[L_{n+1}(\hat{\lambda}') \mid E]] \le \mathbb{E}[\alpha] = \alpha.
\end{equation*}
This completes the proof of Theorem \ref{thm:risk_control}.
\end{proof}

\subsection{Proof of Theorem \ref{thm:lower_bound}}
\label{app:proof_th_risk_lower}

The proof of Theorem \ref{thm:lower_bound} relies on the following lemma regarding the approximate continuity of the empirical risk.

% We first introduce the \textit{jump function}, which quantifies the magnitude of the discontinuity of a right-continuous function $f$ at point $\lambda$:
% \begin{equation}
%     J(f, \lambda) = \lim_{\epsilon \to 0^+} f(\lambda - \epsilon) - f(\lambda).
% \end{equation}
% For a continuous function, the jump is zero. For a step function, it represents the height of the step.

% \begin{lemma}[Jump Lemma]
% \label{lem:jump}
% In the setting of RACER, assume that the non-conformity scores follow a continuous distribution, such that $\mathbb{P}(s(\bm{X}, m) = \lambda) = 0$ for any fixed $\lambda$. Let $\hat{R}_{n+1}(\lambda)$ be the empirical risk on the augmented dataset of size $n+1$. Then, the jumps in the empirical risk are bounded by the weight of a single sample:
% \begin{equation}
%     \sup_{\lambda} J(\hat{R}_{n+1}, \lambda) \le \frac{1}{n+1} \quad \text{almost surely.}
% \end{equation}
% \end{lemma}

\begin{lemma}[Jump Lemma]
\label{lem:jump}
In the setting of RACER, assume that the non-conformity scores follow a continuous distribution, such that $\mathbb{P}(s(\bm{X}, G'(\bm{X})) = \lambda) = 0$ for any fixed $\lambda$. Let $\hat{R}_{n+1}(\lambda)$ be the empirical risk on the augmented dataset of size $n+1$. Then, the magnitude of any discontinuity in the empirical risk is bounded by a single sample's weight:
\begin{equation*}
    \sup_{\lambda} \left( \lim_{\epsilon \to 0^+} \hat{R}_{n+1}(\lambda - \epsilon) - \hat{R}_{n+1}(\lambda) \right) \le \frac{1}{n+1} \quad \text{almost surely.}
\end{equation*}
\end{lemma}

\begin{proof}[Proof of Lemma \ref{lem:jump}]
Recall that the RACER loss function is an indicator $L_i(\lambda) \in \{0, 1\}$. The empirical risk is given by $\hat{R}_{n+1}(\lambda) = \frac{1}{n+1}\sum_{i=1}^{n+1} L_i(\lambda)$.
Since the loss is monotonic, a discontinuity occurs at $\lambda$ only if $L_i(\lambda)$ changes value at $\lambda$ for some sample $i$. This happens precisely when $\lambda$ equals the non-conformity score of that sample.
Since the scores are drawn from a continuous distribution, the probability that two samples have the exact same score is zero (i.e., $\mathbb{P}(s_i = s_j) = 0$ for $i \neq j$). Therefore, almost surely, at any threshold $\lambda$, at most one sample's loss function changes value.
Consequently, the maximum change in the sum $\sum L_i(\lambda)$ is 1, and the difference between the left limit and the value at $\lambda$ for the average $\hat{R}_{n+1}(\lambda)$ is bounded by $\frac{1}{n+1}$.
\end{proof}

% The proof follows the strategy of Theorem 2 in \citet{angelopoulos2021learn}.
\begin{proof}[Proof of Theorem~\ref{thm:lower_bound}]
Let $L_i(\lambda) = l(\bm{X}_i, G'(\bm{X}_i); \lambda)$ denote the loss on the $i$-th calibration point. Recall that in the RACER framework, the loss is an indicator function, bounded by $1$, and is non-increasing with respect to $\lambda$ (Lemma~\ref{lem:loss}).
Let $\hat{R}_{n+1}(\lambda) = \frac{1}{n+1}\sum_{i=1}^{n+1} L_i(\lambda)$.

Consider an auxiliary threshold $\hat{\lambda}'$ computed using the $n+1$ points:
\begin{equation}
\label{eq:proof_lambda_1}
    \hat{\lambda}'' = \inf \left\{ \lambda : \hat{R}_{n+1}(\lambda) + \frac{1}{n+1} \le \alpha \right\}.
\end{equation}

We first establish the relationship between $\hat{\lambda}$ and $\hat{\lambda}''$. Observe that
\begin{equation*}
    \hat{R}_{n+1}(\lambda) = \frac{n}{n+1}\hat{R}_n(\lambda) + \frac{L_{n+1}(\lambda)}{n+1} \ge \frac{n}{n+1}\hat{R}_n(\lambda),
\end{equation*}
since $L_{n+1}(\lambda) \ge 0$. Adding $\frac{1}{n+1}$ to both sides, we obtain
\begin{equation*}
    \hat{R}_{n+1}(\lambda) + \frac{1}{n+1} \ge \frac{n}{n+1}\hat{R}_n(\lambda) + \frac{1}{n+1}.
\end{equation*}
Since each $L_i(\lambda)$ is a non-increasing function of $\lambda$. 
If a specific $\lambda$ satisfies the condition for $\hat{\lambda}$ (i.e., LHS of Eq.~\eqref{eq:proof_lambda_1} $\le \alpha$), the inequality above implies that $\frac{n}{n+1}\hat{R}_n(\lambda) + \frac{1}{n+1} \le \alpha$. Thus, $\hat{\lambda}'' \ge \hat{\lambda}$  almost surely. Since the loss function $L_{n+1}(\lambda)$ is non-increasing, it follows that
\begin{equation*}
    \mathbb{E}[L_{n+1}(\hat{\lambda})] \ge \mathbb{E}[L_{n+1}(\hat{\lambda}'')].
\end{equation*}

By the exchangeability of the calibration and test data, $\hat{\lambda}''$ is symmetric with respect to all $n+1$ points. Thus, the expected loss on the test point equals the expected empirical risk:
\begin{equation*}
    \mathbb{E}[L_{n+1}(\hat{\lambda}'')] = \mathbb{E}\left[ \frac{1}{n+1}\sum_{i=1}^{n+1} L_i(\hat{\lambda}'') \right] = \mathbb{E}[\hat{R}_{n+1}(\hat{\lambda}'')].
\end{equation*}

Next, we focus on lower bound $\hat{R}_{n+1}(\hat{\lambda}'')$. Let $\mathcal{S} = \{ \lambda : \hat{R}_{n+1}(\lambda) + \frac{1}{n+1} \le \alpha \}$. 
By the definition of the infimum, for any $\epsilon > 0$, the value $\hat{\lambda}'' - \epsilon$ does not belong to $\mathcal{S}$. This implies:
\begin{equation*}
    \hat{R}_{n+1}(\hat{\lambda}'' - \epsilon) + \frac{1}{n+1} > \alpha.
\end{equation*}

Next, we invoke the Lemma~\ref{lem:jump} (\textit{Jump Lemma}), which bounds the size of discontinuities in the empirical risk. Given the assumption that the non-conformity scores follow a continuous distribution, the maximum jump size of $\hat{R}_{n+1}$ at any $\lambda$ is bounded by the weight of a single sample, i.e., $1/(n+1)$.
Combining this with the left limit as $\epsilon \to 0^+$, we obtain:
\begin{equation*}
\begin{aligned}
    \hat{R}_{n+1}(\hat{\lambda}'') &\ge \lim_{\epsilon \to 0^+} \hat{R}_{n+1}(\hat{\lambda}'' - \epsilon) - \frac{1}{n+1} \\
     &\ge \left(\alpha - \frac{1}{n+1}\right) - \frac{1}{n+1} \\
     &= \alpha - \frac{2}{n+1}.
\end{aligned}
\end{equation*}

Combining all parts, we conclude:
\begin{equation*}
    \mathbb{E}[L_{n+1}(\hat{\lambda})] \ge \mathbb{E}[\hat{R}_{n+1}(\hat{\lambda}'')] \ge \alpha - \frac{2}{n+1}.
\end{equation*}
This completes the proof of Theorem \ref{thm:lower_bound}.
\end{proof}

\section{Implementation details}
\label{app:implementation_details}

In this section, we provide detailed specifications regarding the weighting schemes, hyperparameter tuning, and numerical smoothing techniques employed in our experiments.

\paragraph{Ground truth set construction.}
To construct ground truth labels, we generated responses for all candidate models using the Language Model Evaluation Harness~\cite{eval-harness} and assigned binary correctness labels ($1$ for correct, $0$ otherwise) based on task-specific metrics (e.g., Exact Match for GSM8K). The collection of models labeled $1$ constitutes the ground truth set, serving as the uniform standard for calibrating and evaluating RACER, independent of the specific supervision signals used during base router training. Importantly, RACER is agnostic to this specific definition and can seamlessly adapt to alternative criteria, such as self-consistency across multiple sampling paths.

\paragraph{Data partitioning.}
To train base routers, we construct a global training set $\mathcal{D}_{\text{train}}$ by sampling from each benchmark's original training split. 
Specifically, for GSM8K, we randomly sample $50\%$ of the original training data to form $\mathcal{D}_{\text{train}}$; for MMLU, CMMLU, and ARC-Challenge, we randomly sample $40\%$ of each benchmark's original data to form $\mathcal{D}_{\text{train}}$.
We then merge all sampled subsets to form $\mathcal{D}_{\text{train}}$ for base-router training.
After training, we apply our RACER paradigm as a post-processing step on the trained routers. 

The remaining data (which constitutes 50\% of the original data for GSM8K and 60\% for others) is then partitioned to evaluate the RACER method. Specifically, this remaining subset is divided into calibration, validation, and test sets.
For most datasets, we allocate 50\% for calibration, 10\% for validation, and 40\% for final testing. However, to account for the smaller sample size of the ARC-Challenge dataset, we adjust these ratios to 40\% for calibration and 20\% for validation, while maintaining 40\% for the test set.

\paragraph{Setting of base router.}
The specific configurations for the three baseline routers are as follows:
\begin{itemize}[itemsep=3pt, parsep=0pt, topsep=1pt]
    \item KNNR \cite{hu2024routerbench}: A k-nearest neighbors (KNN) classifier is trained as the router, using the classification probabilities of each LLM as the router score. We employ cosine similarity as the distance metric and set the number of neighbors to $k=40$.
    \item MLPR \cite{huang2025routereval}: A multi-layer perceptron (MLP) classifier is trained as the router to predict model performance score. The architecture consists of a hidden layer with $256$ units. Training is performed using BCEWithLogitsLoss with a batch size of $32$ and a learning rate of $10^{-4}$ for $100$ epochs.
    \item RouterDC \cite{chen2024routerdc}: RouterDC consists of an encoder and LLM embeddings, utilizing dual contrastive learning to train the router. The LLM embedding dimension is set to $768$. We set the hyper-parameters $\{K_+, K_-, H, \lambda\}$ to $\{3, 3, 3, 1\}$, respectively, and the number of clusters to $N=5$. The model is optimized using AdamW~\cite{loshchilov2018decoupled} for $1000$ steps with a learning rate of $5 \times 10^{-5}$, a weight decay of $0.01$, and a mini-batch size of $32$.
\end{itemize}
We utilize mDeBERTav3-base~\cite{he2021deberta} as the shared underlying encoder for all base routers across our benchmarks. All experiments are run on NVIDIA A100 and NVIDIA L40 GPUs. 

\paragraph{Non-conformity scores.}
To evaluate the adaptability of the RACER paradigm, we use two distinct non-conformity score functions derived from the augmented router scores \(r(\bm{x}, m)\):
\begin{enumerate}[itemsep=3pt]
	\item \textit{Router Score-Gap}. 
	$s_{\text{gap}}(\bm{x}, m) = r_{\max}(\bm{x}) - r(\bm{x}, m)$,
	where \(r_{\max}(\bm{x})=\max_{m \in \mathcal{M}'}r(\bm{x},m)\).
    This metric captures the confidence gap between $m$ and the top-ranked model. Smaller values indicate that $m$ is close to the maximum score.
	\item \textit{Inverse Probability}. 
	$s_{\text{prob}}(\bm{x}, m) = 1 - r(\bm{x}, m)$. Here, a lower score directly corresponds to higher confidence in selecting model $m$.
\end{enumerate}
These functions encode complementary signals: \textit{router score-gap} emphasizes relative separation, whereas \textit{inverse probability} leverages absolute confidence. Jointly, they allow us to empirically validate the robustness of RACER across varying non-conformity definitions.

To ensure theoretical validity and numerical distinctness, we incorporate a randomized smoothing step by defining the final score as $\tilde{s}(\bm{x}, m) = s(\bm{x}, m) + \epsilon$, where \(\epsilon \sim \mathcal{U}[0, 10^{-6}]\). 
Theoretically, this continuous noise satisfies the continuity assumption (i.e., \(\mathbb{P}(s(\bm{X}, m) = \lambda) = 0\)) required by Theorem \ref{thm:lower_bound} to guarantee the risk lower bound. 
Practically, since the magnitude of $\epsilon$ is negligible, this perturbation effectively resolves ties among models with identical router scores without altering the relative ranking of models with distinct scores.

\paragraph{Weighting schemes.}
In the weighted aggregation method, we use the following three metrics to determine the unnormalized weight $w_m$ for each model:
\begin{itemize}[itemsep=3pt, parsep=0pt, topsep=1pt]
    \item \textit{Base router score.} The weight is directly derived from the router's probability:  $w_m = r(\bm{x}_{n+1}, m)$. This reflects the router's intrinsic confidence in assigning the query to model $m$.
    \item \textit{Verbal binary confidence} \cite{lin2022teaching}. After generating the initial answer \( a_m \), we concatenate the query and answer with the prompt in Figure~\ref{fig:prompt_confidence} and run the model again. We instruct the model to output a single token (0 or 1). The weight \( w_m \) is set to this binary output (i.e., \( w_m \in \{0, 1\} \)).
    \item $\bm P(\text{True})$ \cite{kadavath2022language}. Similar to the verbal binary method, we use the prompt in Figure~\ref{fig:prompt_confidence} to assess correctness. However, instead of using the discrete output token, we extract the probability assigned to the token ``1" (representing high confidence) and use this soft probability as the weight \( w_m \).
\end{itemize}
For the \textit{Verbal binary confidence} and $\bm P(\text{True})$ metrics, relying on each model's self-evaluation could introduce inconsistencies due to varying calibration capabilities. To ensure that confidence scores are comparable across the router model set, we employ a unified evaluator strategy. Specifically, for each dataset, we select the single best-performing model as reported in \cite{chen2024routerdc} to compute the confidence scores for all answers generated by the candidate models.

\begin{figure}[h]
    \centering
    % 这里开始 tcolorbox
    \begin{tcolorbox}[
        colback=gray!5,
        colframe=black!75,
        title=\textbf{Prompt Template for Confidence Extraction}, % 盒子自带的顶部标题栏
        sharp corners=south,
        boxrule=0.8pt,
        left=5pt, right=5pt, top=5pt, bottom=5pt,
        fonttitle=\bfseries
    ]
        Question: \texttt{\{question\}} \\
        Proposed Answer: \texttt{\{answer\}} \\
        \vspace{0.2cm} %稍微加点间距
        Now I will rate my confidence in the proposed answer as either (0) or (1). \\
        Proposed confidence: (
    \end{tcolorbox}
    % tcolorbox 结束

    \vspace{-0.2cm} % 调整盒子和Caption的间距
    % Caption 放在盒子外面
    \caption{The prompt template used for extracting model confidence. The placeholders \texttt{\{question\}} and \texttt{\{answer\}} are replaced by the input query \( \bm{x}_{n+1} \) and the model's generated response \( a_m \), respectively.}
    \label{fig:prompt_confidence}
\end{figure}

\paragraph{Metrics.}
To evaluate the effectiveness of the RACER paradigm, we employ three key metrics: average \textit{Risk}, \textit{Size} and \textit{Accuracy}. These metrics are formally defined as:
\begin{equation*}
	\begin{aligned}
		& \text{Risk} := \frac{1}{N}\sum_{i=1}^{N} \mathbf{1}(C_{\hat \lambda}(\bm{x}_i)\cap G'_i = \emptyset), \\
        & \text{Size} := \frac{1}{N}\sum_{i=1}^{N} \left| C_{\hat \lambda}(\bm{x}_i) \cap \mathcal{M} \right|,\\
		& \text{Accuracy} := \frac{1}{N}\sum_{i=1}^{N}\mathbf{1}(\hat{y}_i =y^*_i ).
	\end{aligned}
\end{equation*}
where $N$ is the number of test queries, $G'_i$ denotes the augmented ground truth set, $\hat{y}_i$ and $y^*_i$ represent the final aggregated output of RACER and the correct answer, respectively. 
Crucially, the definition of \textit{size} uses the intersection with $\mathcal{M}$ to explicitly exclude the virtual null model $m_\emptyset$. This ensures the metric faithfully reflects the actual inference overhead by counting only the candidate LLMs.

\paragraph{Temperature scaling.} 
The temperature parameter \( T \) in the softmax function controls the entropy of the weight distribution. A higher \( T \) yields a more uniform distribution (approaching simple majority voting), while a lower \( T \) sharpens the distribution, giving significantly more influence to models with higher confidence scores. We optimize \( T \) to maximize performance on the validation set.

\paragraph{Selection of optimal configurations.}
To rigorously evaluate the performance of RACER-G and RACER-P on downstream tasks, we treat the target risk level $\alpha$ and the specific aggregation strategy as hyperparameters to be tuned.
We determine the optimal configuration by performing a grid search on the validation set, identifying the combination of $\alpha$ and aggregation method (e.g., Majority Voting or specific Weighted Aggregation metrics) that maximizes accuracy.
The optimal $\alpha$ and aggregation strategy identified on the validation set are then applied unchanged to the test set to derive the final reported results.
This protocol ensures that the performance improvements reported in our experiments are achieved through a rigorous validation process, avoiding overfitting to the test data.

\section{Additional experimental results}
\label{app:full_results}
In this section, we present the comprehensive experimental results that complement the findings in the main text. We specifically focus on the trial-wise distributions of risk and set size at the standard target level ($\alpha=0.1$), followed by a detailed analysis of the downstream accuracy and hyperparameter selection.

\subsection{Detailed results on risk control and model set size}
\label{app:detail_risk_size}
Figure~\ref{fig:racer_risk_size} presents the complete histograms for empirical risk and prediction set size across all four benchmarks (GSM8K, MMLU, CMMLU, ARC-Challenge) and three base routers (RouterDC, MLPR, KNNR) over 100 independent trials.

\begin{figure}[!tbhp]
	\centering
	\begin{subfigure}{\linewidth}
		\centering
		\includegraphics[width=\linewidth]{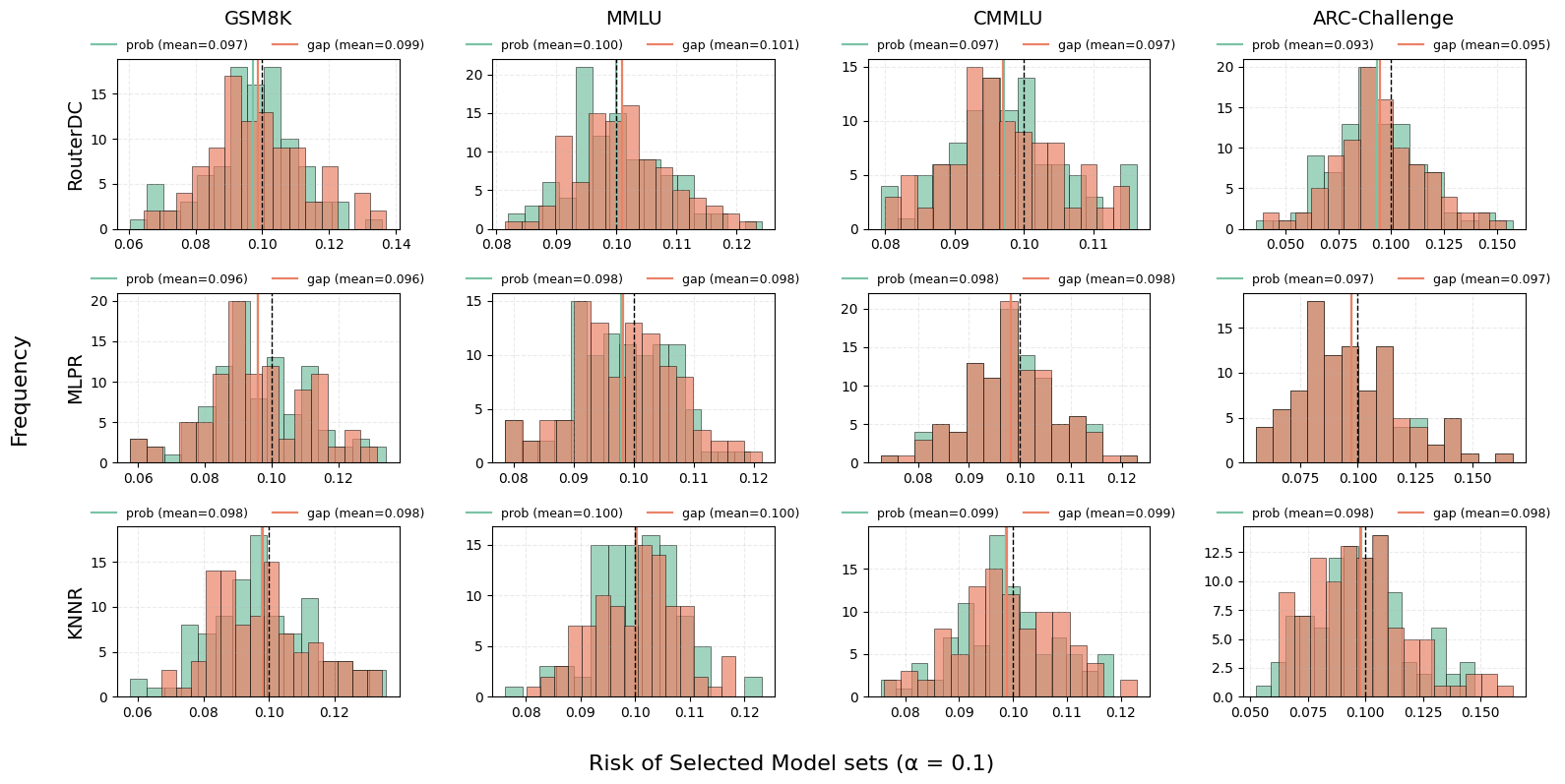}
		% \caption{\textit{Risk} control of RACER.}
	\end{subfigure}
	\vspace{0.5em} % 控制上下两图的间距，可微调
	\begin{subfigure}{\linewidth}
		\centering
		\includegraphics[width=\linewidth]{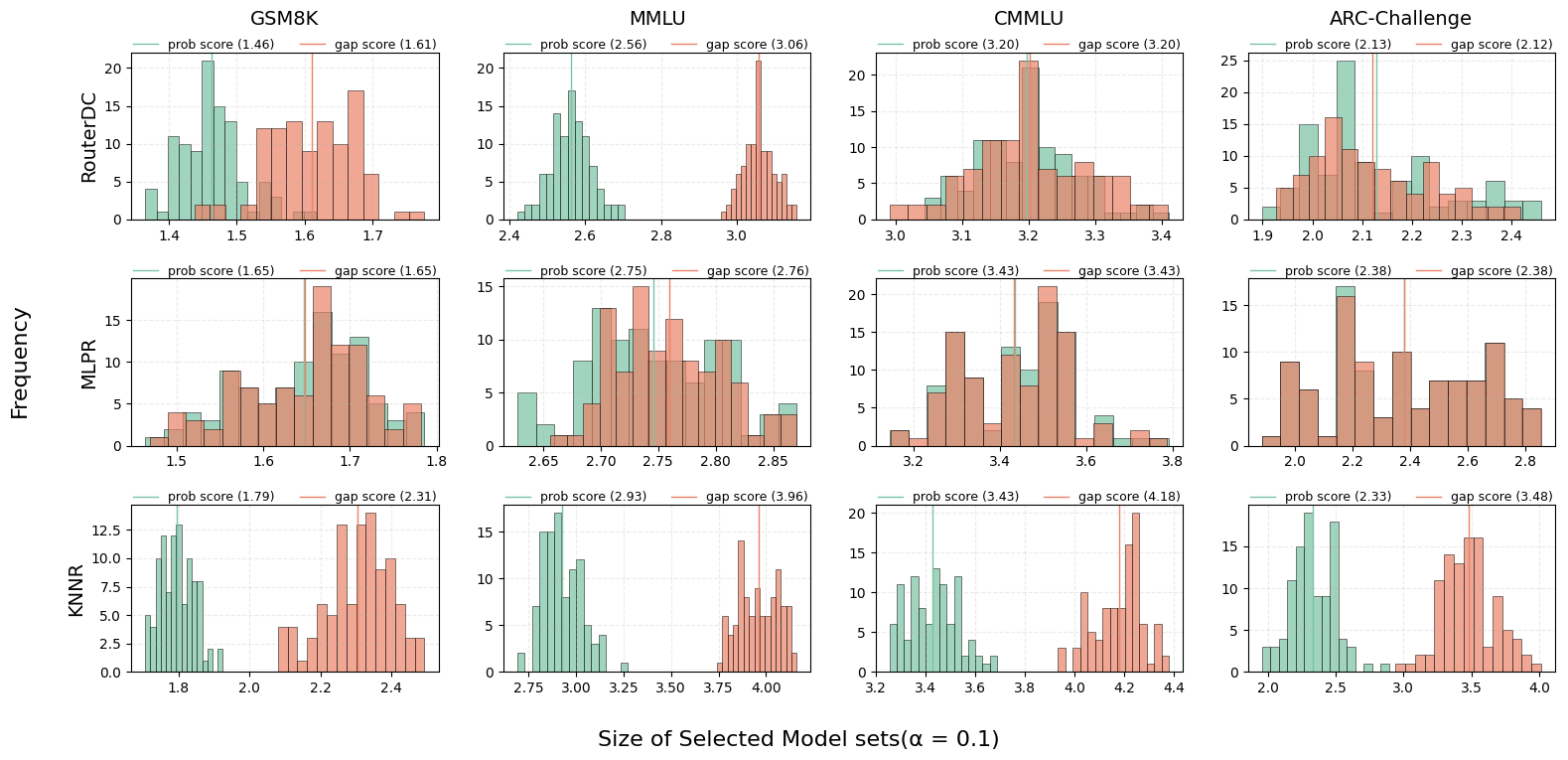}
		% \caption{\textit{Set size} of RACER.}
	\end{subfigure}
	\caption{\textbf{Distributions of empirical \textit{Risk} and \textit{Size} over 100 independent trials with a target risk level $\alpha=0.1$.} The top row displays the risk distribution, where the vertical black dashed line indicates the target risk level $\alpha$. The bottom row illustrates the distribution of prediction set sizes. The \textcolor[rgb]{0.4,0.7,0.6}{green} bars represent the \textit{router score-gap} non-conformity score, while the \textcolor[rgb]{0.9,0.6,0.5}{orange} bars represent the \textit{inverse probability} non-conformity score. The results empirically demonstrate that RACER strictly controls the risk around the target level across different base routers and benchmarks.}
 \label{fig:racer_risk_size}
\end{figure}

\paragraph{RACER achieves precise risk control across diverse benchmarks.}
Consistent with the representative results on CMMLU shown in the main text, the top row of Figure~\ref{fig:racer_risk_size} confirms that RACER maintains rigorous risk control across all evaluated datasets. The empirical risk distributions are consistently centered around the target vertical line ($\alpha=0.1$), with narrow fluctuations. This validates that the calibration guarantee of RACER is distribution-free and holds regardless of the underlying base router's architecture or the task's difficulty.

\paragraph{Impact of base router and non-conformity score on set size.}
The bottom row of Figure~\ref{fig:racer_risk_size} illustrates the distribution of prediction set sizes. We clearly observe that stronger base routers (e.g., RouterDC) naturally lead to more compact sets compared to weaker ones (e.g., KNNR), as the correct answer is ranked higher. Furthermore, comparing the two scoring methods, the \textit{inverse probability} score (orange bars) frequently yields smaller sets than the \textit{router score-gap} (green bars), particularly on weaker base routers. This demonstrates that probability-based uncertainty quantification offers better statistical efficiency at the same risk level.

\subsection{Detailed results on downstream accuracy}
\label{app:detail_acc}

In this section, we provide a deeper analysis of the accuracy improvements reported in Section~\ref{subsec:result}. Specifically, we examine the optimal hyperparameters (i.e., the risk level $\alpha$ and the aggregation strategy) that were selected to maximize performance on validation set.
To maintain brevity, we define the abbreviations of aggregation methods used in the subsequent analysis as follows:
    \begin{itemize}[itemsep=3pt, parsep=0pt, topsep=1pt]
        \item \textbf{Majority} (Majority Voting): Unweighted voting where the answer predicted by the largest number of selected candidate models is chosen.
        \item \textbf{W-Router} (Weighted by Router Score): Weighted voting using the router's predicted score as the confidence weight.
        \item \textbf{W-Binary} (Weighted by Verbal Binary Confidence): Weighted voting using the model's self-reported verbal confidence.
        \item \textbf{W-$\bm P(\text{True})$} (Weighted by $P(\text{True})$): Weighted voting using the probability of the token ``True''.
    \end{itemize}

\begin{figure}[t]
    \centering
    \includegraphics[width=0.8\linewidth]{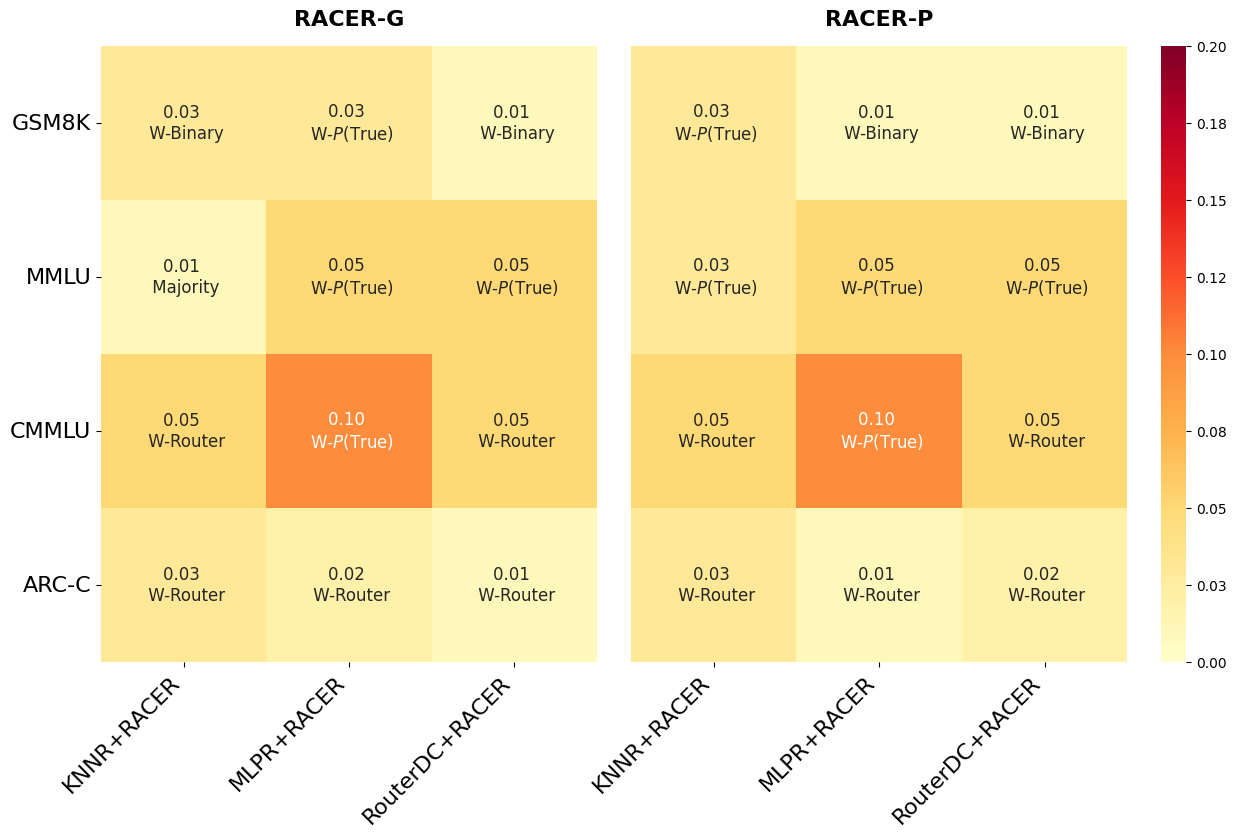}
    \caption{\textbf{Heatmap of optimal hyperparameters selected on the validation set.} The figure displays the grid-search results for the target risk level $\alpha$ (indicated by the cell value and color intensity) and the aggregation method (text annotation) that achieved the highest accuracy on the validation set. The left panel shows configurations for RACER-G (router score-gap), and the right panel for RACER-P (inverse probability). These optimal configurations were applied to the test set to produce the main results in Table~\ref{tab:main_accuracy}.}
    \label{fig:alpha_heatmap}
\end{figure}

\begin{figure}
    \centering
    \includegraphics[width=1\linewidth]{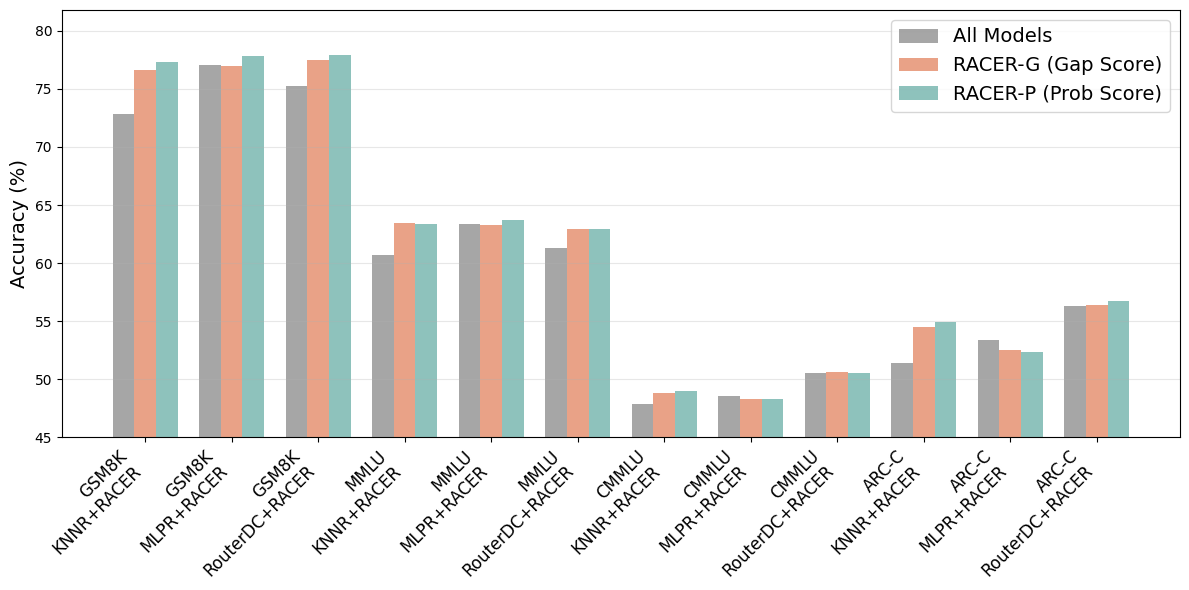}
    \caption{\textbf{Performance evaluation of RACER versus full model aggregation.} The figure presents the comparison of RACER-G/P (w/ Agg.) against the full model aggregation baseline across varying dataset-router pairs. RACER-P demonstrates superior performance in 10/12 cases, and RACER-G in 8/12 cases, thereby confirming that effectively filtering noisy models leads to more robust predictions than simply aggregating all available models.}
    \label{fig:vs_full_models}
\end{figure}

\paragraph{Optimal risk tolerance is data-dependent.}
Figure~\ref{fig:alpha_heatmap} illustrates that the optimal $\alpha$ is not static but dynamically adapts to the nature of the task. As shown in the heatmap, reasoning-intensive benchmarks (e.g., GSM8K, ARC-C) consistently favor strict risk control (typically $\alpha \in [0.01, 0.03]$) to enforce high-precision retrieval. In contrast, for broad knowledge tasks (e.g., CMMLU), looser constraints ($\alpha \approx 0.05 - 0.10$) are preferred to maximize downstream accuracy.

\paragraph{Weighted aggregation dominates majority voting.}
Figure~\ref{fig:alpha_heatmap} demonstrates that confidence-aware aggregation consistently outperforms simple majority voting. Furthermore, the heatmap reveals that continuous weighting schemes, specifically \textit{Router Score} and $P(\text{True})$, are selected as optimal significantly more often than the discrete \textit{Binary} confidence. This indicates that fine-grained confidence signals provide superior granularity for distinguishing model reliability compared to coarse binary (0/1) indicators, a finding that aligns with recent observations in \cite{taubenfeld2025confidence}.

% \begin{longtable}{llcccccccc}

\begin{table*}[t]
\centering
\caption{\textbf{\textit{Risk} across varying target risk levels ($\alpha$).} We report the empirical risk achieved by different RACER configurations (combining three base routers with two non-conformity scores) over 100 independent trials. The results span four benchmarks (GSM8K, MMLU, CMMLU, ARC-Challenge) and a range of user-specified risk level $\alpha \in [0.01, 0.50]$. The close alignment between the \textit{risk} (values in the table) and the target $\alpha$ (column headers) confirms the validity of the RACER calibration procedure across a wide range of risk levels.}
\label{tab:risk_vs_alpha_long}
\renewcommand\arraystretch{1.1}
\resizebox{1.00\textwidth}{!}{%
    \setlength{\tabcolsep}{5mm}{
    % \small
    % \footnotesize
    \begin{tabular}{llcccccccc}
\toprule
\multirow{2}{*}{Dataset} & \multirow{2}{*}{Method} &
\multicolumn{8}{c}{$\alpha$} \\
\cmidrule(lr){3-10}
& & 0.01 & 0.05 & 0.10 & 0.15 & 0.20 & 0.30 & 0.40 & 0.50 \\
\midrule

% ===================== gsm8k =====================
\multirow{6}{*}{\textbf{GSM8K}} & KNNR+RACER-G     & 0.010 & 0.048 & 0.098 & 0.147 & 0.197 & 0.297 & 0.394 & 0.495 \\
                               & KNNR+RACER-P     & 0.009 & 0.048 & 0.098 & 0.147 & 0.196 & 0.296 & 0.395 & 0.495 \\
                               & MLPR+RACER-G     & 0.010 & 0.048 & 0.096 & 0.145 & 0.195 & 0.297 & 0.396 & 0.497 \\
                               & MLPR+RACER-P     & 0.010 & 0.048 & 0.096 & 0.145 & 0.195 & 0.296 & 0.396 & 0.497 \\
                               & RouterDC+RACER-G & 0.009 & 0.048 & 0.099 & 0.146 & 0.195 & 0.296 & 0.394 & 0.493 \\
                               & RouterDC+RACER-P & 0.010 & 0.049 & 0.097 & 0.145 & 0.195 & 0.296 & 0.394 & 0.493 \\
\midrule

% ===================== mmlu =====================
\multirow{6}{*}{\textbf{MMLU}}  & KNNR+RACER-G     & 0.010 & 0.050 & 0.100 & 0.151 & 0.200 & 0.301 & 0.400 & 0.500 \\
                               & KNNR+RACER-P     & 0.010 & 0.050 & 0.100 & 0.150 & 0.200 & 0.300 & 0.400 & 0.500 \\
                               & MLPR+RACER-G     & 0.010 & 0.048 & 0.098 & 0.149 & 0.199 & 0.298 & 0.397 & 0.498 \\
                               & MLPR+RACER-P     & 0.010 & 0.049 & 0.098 & 0.148 & 0.199 & 0.298 & 0.397 & 0.498 \\
                               & RouterDC+RACER-G & 0.010 & 0.049 & 0.101 & 0.151 & 0.201 & 0.300 & 0.399 & 0.499 \\
                               & RouterDC+RACER-P & 0.010 & 0.050 & 0.100 & 0.150 & 0.200 & 0.300 & 0.401 & 0.500 \\
\midrule

% ===================== cmmlu =====================
\multirow{6}{*}{\textbf{CMMLU}} & KNNR+RACER-G     & 0.010 & 0.050 & 0.099 & 0.149 & 0.199 & 0.298 & 0.398 & 0.499 \\
                               & KNNR+RACER-P     & 0.010 & 0.050 & 0.099 & 0.147 & 0.197 & 0.298 & 0.398 & 0.499 \\
                               & MLPR+RACER-G     & 0.010 & 0.048 & 0.098 & 0.151 & 0.201 & 0.300 & 0.399 & 0.500 \\
                               & MLPR+RACER-P     & 0.010 & 0.048 & 0.098 & 0.151 & 0.201 & 0.300 & 0.399 & 0.500 \\
                               & RouterDC+RACER-G & 0.010 & 0.048 & 0.097 & 0.147 & 0.198 & 0.297 & 0.399 & 0.500 \\
                               & RouterDC+RACER-P & 0.010 & 0.048 & 0.097 & 0.146 & 0.197 & 0.297 & 0.400 & 0.499 \\
\midrule

% ===================== arc_challenge =====================
\multirow{6}{*}{\textbf{ARC-C}} & KNNR+RACER-G     & 0.008 & 0.048 & 0.098 & 0.146 & 0.194 & 0.294 & 0.391 & 0.493 \\
                                        & KNNR+RACER-P     & 0.008 & 0.048 & 0.098 & 0.145 & 0.194 & 0.294    & 0.393 & 0.494 \\
                                        & MLPR+RACER-G     & 0.008 & 0.048 & 0.097 & 0.146 & 0.197 & 0.296 & 0.395 & 0.497 \\
                                        & MLPR+RACER-P     & 0.008 & 0.048 & 0.097 & 0.147 & 0.197 & 0.296 & 0.395 & 0.497 \\
                                        & RouterDC+RACER-G & 0.008 & 0.045 & 0.095 & 0.142 & 0.192 & 0.294 & 0.396 & 0.495 \\
                                        & RouterDC+RACER-P & 0.007 & 0.046 & 0.093 & 0.142 & 0.192 & 0.296 & 0.396 & 0.496 \\
\bottomrule
\end{tabular}
}
}
\end{table*}
% \end{longtable}

\paragraph{Consistency across non-conformity measures.}
Comparing the left (RACER-G) and right (RACER-P) panels in Figure~\ref{fig:alpha_heatmap}, we observe a striking alignment in the distribution of optimal hyperparameters. For most dataset-router pairs, the selected $\alpha$ levels and aggregation strategies remain largely invariant to the choice of the non-conformity score. This consistency suggests that the optimal configuration is primarily driven by the intrinsic properties of the task and the base router rather than the specific formulation of the uncertainty metric, further highlighting the robustness of the RACER framework.

\subsection{Detailed results on efficiency and accuracy gains over full model aggregation}
\label{app:vs_full_model}
% In this section, we evaluate the benefits of RACER's routing paradigm compared to the \textit{full model aggregation} baseline. To ensure a rigorous comparison, the aggregation strategy for the baseline was also optimized on the validation set to achieve its best possible performance. Consequently, any performance gains observed in RACER are derived from its effective filtering mechanism rather than a sub-optimal baseline.
This section serves as a supplement to the efficiency analysis in Figure~\ref{fig:model_calls_save} of the main text. We evaluate the accuracy of RACER compared to the \textit{full model aggregation} baseline. To ensure a rigorous comparison, the aggregation strategy for the baseline was also optimized on the validation set.

\paragraph{Aggregation over selected model sets mitigates noise.}
Contrary to the intuition that ``more models yield better results'', Figure~\ref{fig:vs_full_models} reveals that aggregating all available models, even when the aggregation parameters are optimized, often introduces noise from weaker predictors. By performing aggregation over the selected model sets, RACER constructs a cleaner candidate pool. Specifically, RACER-P outperforms the optimized full model baseline in \textbf{10 out of 12} experimental configurations, and RACER-G surpasses it in \textbf{8 out of 12} cases. This confirms that excluding irrelevant models via risk-aware routing is more effective than simply assigning them low weights in a full aggregation scheme.

\section{Extensive study}
\label{app:extensive_study}
In this section, we extend our evaluation to investigate the behavior of RACER under varying $\alpha$. We analyze how the risk control stability and prediction set efficiency evolve as the user-specified risk tolerance $\alpha$ ranges from $0.01$ to $0.50$.

\subsection{Performance of RACER under varying risk levels}
\label{app:ex_multiple_alpha}
% Table~\ref{tab:risk_vs_alpha_long} reports the precise mean empirical risk across a wide range of $\alpha$, and Figure~\ref{fig:size_vs_alpha} visualizes the corresponding evolution of prediction set sizes.

\paragraph{Robust risk control across a wide range of risk levels.}
As shown in Table~\ref{tab:risk_vs_alpha_long}, the mean empirical risk aligns precisely with the user-specified risk level $\alpha$ across all ranges. For instance, at $\alpha=0.10$, the empirical risk across all datasets and methods fluctuates narrowly between $0.093$ and $0.101$. This tight alignment persists from strict safety requirements ($\alpha=0.01$) to looser constraints ($\alpha=0.50$), demonstrating that RACER's calibration validity is robust to the choice of $\alpha$, the underlying dataset difficulty, and the base router architecture.

\begin{figure}[t]
    \centering
    \includegraphics[width=\linewidth]{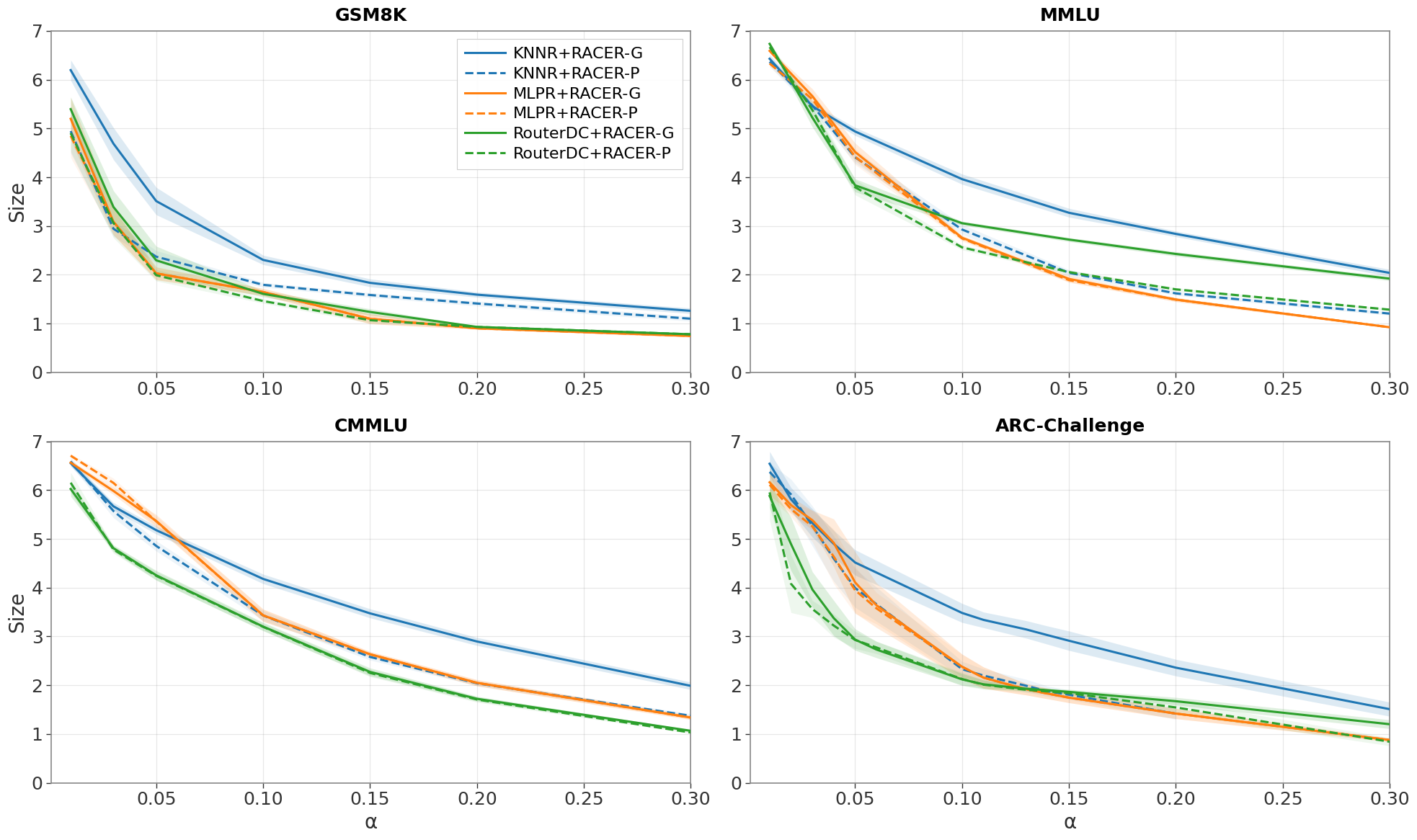}
    \caption{\textbf{Evolution of average selected model set size with respect to the target risk level $\alpha$.} The plots display the mean set size computed over 100 independent trials across four diverse benchmarks (GSM8K, MMLU, CMMLU, and ARC-Challenge) as the user-specified risk level varies within the range $\alpha \in [0.01, 0.3]$. Different colors distinguish the base routers: \textcolor[rgb]{0.12,0.47,0.71}{Blue} for KNNR, \textcolor[rgb]{1.0,0.5,0.05}{Orange} for MLPR, and \textcolor[rgb]{0.17,0.63,0.17}{Green} for RouterDC. Line styles differentiate the non-conformity scores: solid lines correspond to RACER-G (router score-gap) and dashed lines to RACER-P (inverse probability). Shaded regions indicate the standard deviation. The plots illustrate a monotonic trade-off where higher risk allows for more compact prediction sets, with RouterDC consistently demonstrating superior efficiency (smallest sizes) across a wide range of risk levels.}    \label{fig:size_vs_alpha}
\end{figure}

\paragraph{Prediction set size dynamically reflects task difficulty and risk trade-offs.}
As illustrated in Figure~\ref{fig:size_vs_alpha}, the average selected model set size decreases monotonically as the risk tolerance $\alpha$ increases. This trend is consistent with the nested nature of the decision sets: a higher user-specified risk level allows the router to be more selective, naturally resulting in smaller model sets. Crucially, RACER exhibits desirable adaptivity across different domains. For challenging benchmarks like CMMLU and ARC-Challenge, the system constructs larger prediction sets to maintain strict risk control. In contrast, for tasks with clearer reasoning signals like GSM8K, the set sizes drop more rapidly. This confirms that RACER effectively optimizes inference cost by tailoring the set size to the inherent uncertainty of the query distribution.

\paragraph{Superior base routers and inverse probability score improve statistical efficiency.}
While all configurations satisfy the risk control constraint, their efficiency varies.
First, regarding the base routers, Figure~\ref{fig:size_vs_alpha} shows that RouterDC (green lines) consistently achieves the smallest size compared to MLPR and KNNR at any given $\alpha$. This indicates that a better-performing base ranker allows RACER to construct more compact sets, thereby reducing inference overhead.
Second, regarding the non-conformity scores, the \textit{inverse probability} score (dashed lines) demonstrates superior robustness. As observed in the KNNR plots in Figure~\ref{fig:size_vs_alpha}, the inverse probability variant frequently yields smaller sets than the router score-gap variant (solid lines). This efficiency gap is particularly pronounced on weaker base routers, suggesting that the inverse probability score is more effective at extracting useful uncertainty signals from suboptimal models.

\end{document}